\documentclass[twoside,11pt]{article}

\usepackage{jmlr2e}
\usepackage{amsmath,bbm,amssymb,multirow,natbib,graphicx,makecell,subfigure,booktabs,array,url,mathtools,wrapfig,lipsum,mathrsfs,dsfont,epstopdf,bm,bbm,relsize,caption,color,algorithm,enumitem}

\usepackage[noend]{algpseudocode}
\usepackage[useLove2s,dvipsLove2s,svgLove2s,table]{xcolor}
\usepackage[colorlinks=true,
linkcolor=blue,
urlcolor=blue,
citecolor=blue]{hyperref}

\setcellgapes{5pt}

\newtheorem{conj}{Conjecture}
\newtheorem{thm}[conj]{Theorem}
\newtheorem{cor}[conj]{Corollary}

\newtheorem{ass}{Assumption}

\def\qed{\hfill\BlackBox}
\def\L{\mathcal{L}}
\def\A{\mathcal{A}}

\def\R{\mathbb{R}}
\def\wh{\widehat}
\def\wt{\widetilde}
\def\bI{\bm{I}}
\def\E{\mathcal{E}}
\def\u{\mu}
\def\I{\mathcal{I}}

\def\g{\gamma}
\def\a{\alpha}
\def\EE{\mathbb{E} }
\def\PP{\mathbb{P} }

\def\eps{\varepsilon}
\def\S{\mathcal{S}}
\def\diag{\textrm{diag}}
\def\i{\infty}
\def\r{{\infty,1}}
\def\1{\bm{1}}
\def\M{\Pi}
\def\KL{\textrm{KL}}

\def\ua{\underline{\alpha}}
\def\uu{\underline{\mu}}
\def\og{\overline{\gamma}}
\def\ug{\underline{\gamma}}

\DeclarePairedDelimiter{\floor}{\lfloor}{\rfloor}
\let\emptyset\varnothing
\makeatletter
\def\BState{\State\hskip-\ALG@thistlm}
\makeatother
\title{Optimal estimation of sparse topic models}

\author{Xin Bing\thanks{Department of Statistics and Data Science, Cornell University, Ithaca, NY. E-mail: \texttt{xb43@cornell.edu}.}~~~~~Florentina Bunea\thanks{Department of Statistics and Data Science, Cornell University, Ithaca, NY. E-mail: \texttt{fb238@cornell.edu}.}~~~~~Marten Wegkamp\thanks{Department of Mathematics and Department of Statistics and Data Science, Cornell University, Ithaca, NY. E-mail: \texttt{mhw73@cornell.edu}. }}

\begin{document}
	
	
	\maketitle
	
	\begin{abstract}%
		Topic models have become popular tools for dimension reduction and exploratory analysis of text data which  consists in  observed frequencies of a vocabulary of $p$ words in $n$ documents, stored in a  $p\times n$ matrix. The main premise is that the mean of this data matrix  can be factorized into a product of two non-negative matrices: a $p\times K$ word-topic matrix $A$ and a $K\times n$ topic-document matrix $W$. 
		
		This paper  studies the estimation of $A$ that is possibly element-wise sparse, and the number of topics $K$ is unknown. In this under-explored context, we derive a new minimax lower bound for the estimation of such $A$ and propose a new computationally efficient algorithm for its recovery. 
		We derive a finite sample upper bound for our estimator,  and  show that it matches the minimax lower bound in many scenarios. Our estimate adapts   to the unknown sparsity of $A$ and our analysis is valid for any finite $n$, $p$, $K$ and document lengths. 
		
		Empirical results on both synthetic data and semi-synthetic data show that our proposed estimator is a strong competitor of the existing state-of-the-art algorithms for both non-sparse $A$ and sparse $A$, and has superior performance is many scenarios of interest.  
	\end{abstract}
	
	\section{Introduction}
	Topic modeling has been a popular and powerful statistical model during the last two decades in machine learning and natural language processing for discovering thematic structures from a corpus of documents. Topic models have wide applications beyond the context in which was originally introduced, to genetics, neuroscience and social science \citep{blei-intro}, to name just a few areas in which they have been successfully employed.

	In the computer science and machine learning literature, topic models were first introduced as {\em latent semantic indexing models} by \cite{Deerwester,Papa98, Hofmann99, Papa}.
	For uniformity and clarity, we explain our methodology in the language typically associated with this set-up.
	A corpus of $n$ documents is assumed to follow generative models based on the bag-of-word representation. Specifically, each document $X_i \in \R^p$ is a vector containing  empirical (observed) frequencies of $p$ words from a pre-specified dictionary, generated as 
	\begin{equation}\label{model_multinomial}
	X_i \sim {1\over N_i}\text{Multinomial}_p\left(N_i,  \M_i\right),\quad \text{ for each  }i\in[n] := \{1, 2, \ldots, n\}.
	\end{equation} 
	Here $N_i$ denotes the length (or the number of sampled words) in the $i$th document. 
	The expected frequency vector $\M_i \in \R^p$ is called the word-document vector, and  is  a convex combination of $K$ word-topic vectors with weights corresponding to the allocation of $K$ topics. Mathematically, one postulates that
	\begin{equation}\label{eq_M_i}
	\M_i = \sum_{k = 1}^K A_{\cdot k} W_{ki}
	\end{equation}
	where $A_{\cdot k}=(A_{1k}, \ldots, A_{pk})$ is the word-topic vector for the $k$th topic and $W_{\cdot i} = (W_{1i}, \ldots, W_{Ki})$ is the allocation of $K$ topics in this $i$th document. From a probabilistic point of view,  equation (\ref{eq_M_i}) has the conditional probability interpretation
	\begin{equation}\label{bayes}
		 \underbrace{\PP(\text{word }j\ | \text{ document }i)}_{\M_{ji}} = \sum_{k=1}^K 	\underbrace{\PP(\text{word }j\ | \text{ topic }k)}_{A_{jk}} \cdot  \underbrace{\PP(\text{topic }k\ | \text{ document }i)}_{W_{ki}}
\end{equation}
	for each $j\in [p]$, justified by  Bayes' theorem.
	 As a result, the (expected) word-document frequency matrix $\M = (\M_1, \ldots, \M_n)\in \R^{p\times n}$ has the following decomposition
	 \begin{equation}\label{model}
	 	\M = AW = A(W_1, \ldots, W_n).
	 \end{equation}
	 The entries of the  columns of  $\Pi, A$ and $W$ are probabilities, so they are non-negative and sum to one:
	\begin{equation}\label{col_sum_one}
	\sum_{j=1}^p\M_{ji} = 1, \quad \sum_{j=1}^pA_{jk}=1, \quad \sum_{k=1}^KW_{ki} = 1, \quad \text{for any $k\in [K]$ and  $i\in[n]$.}
	\end{equation}
	 Since the number of topics, $K$, is typically much smaller than $p$ and $n$, the matrix $\M$ exhibits a low-rank structure. In the topic modeling literature, the main interest is to recover the matrix $A$ when only the $p \times n$ frequency matrix $X = (X_1, \ldots, X_n)$ and the  document lengths $N_1, \ldots, N_n$ are observed. 
	
	One direction of a large body of work is of Bayesian nature, and the most commonly used prior distribution on $W$ is  the   Dirichlet distribution \citep{BleiLDA}.  Posterior inference on $A$ is then typically conducted via  variational inference  \citep{BleiLDA}, or sampling techniques involving  MCMC-type solvers \citep{MCMC}. We refer to \cite{blei-intro} for a in-depth review. 
	
	The computational intensive nature of Bayesian approaches, in high dimensions,  motivated a separate line of recent work that develops efficient algorithms, with theoretical guarantees,   from a frequentist perspective. \cite{Anandkumar} proposes an estimation method, with provable guarantees, that employs the third moments of $\M$ via a tensor-decomposition. However, the success of this approach requires the topics not be  correlated, and in many situations   there is strong evidence suggesting the contrary 
	\citep{blei2007, LM-dag}.
	
	This motivated another  line of work, similar in spirit with the work presented in this paper, which relies on the following \emph{separability} condition on $A$,  and allows for correlated topics.
	\begin{ass}[separability]\label{ass_sep}
		For each topic $k\in [K]$, there exists at least one word $j$ such that $A_{jk} >0$ and $A_{j\ell} = 0$ for any $\ell \ne k$.
	\end{ass}
	The  {separability} condition was first introduced by \cite{donohoNMF} to ensure   uniqueness in the  Non-negative Matrix Factorization (NMF) framework.   \cite{arora2012learning} introduce the  {separability} condition to the topic model literature with the interpretation that,  for each topic, there exist some words which  \emph{only} occur in this topic. These special words are called  \emph{anchor words}  \citep{arora2012learning} and guarantee    recovery of $A$,  coupled with the following condition on $W$ \citep{arora2012learning}.
	\begin{ass}\label{ass_pd_W}
		Assume the matrix $ n^{-1}WW^\top $ is strictly positive definite.
	\end{ass}
	\noindent Finding anchor words  is the first step towards the  recovery of the desired target  $A$. Many algorithms are developed for this purpose, see, for instance, \cite{arora2012learning,rechetNMF,arora2013practical,ding, Tracy}. All these works require   the number of topics $K$ be  {\em known}, yet in practice $K$ is rarely known.
	This motivated us \cite{Top} to develop
	a method that estimates $K$ consistently from the data under   the  {\em incoherence} Condition \ref{ass_w} on the topic-document matrix $W$ given in Section \ref{sec_disc_L}.
	We defer to this for further discussion of other existing methods  for finding anchor words.
	
	Despite the wide-spread interest and usage of topic models, most of the existing works are mainly devoted to the   computational aspects of estimation, and
	relatively few works  provide statistical guarantees  for  estimators of $A$. An exception is \cite{arora2012learning, arora2013practical} that provide upper bounds for the $\ell_1$-loss $\|\wh A - A\|_1 = \sum_{j=1}^p \sum_{k = 1}^K|\wh A_{jk}- A_{jk}|$ of their estimator. Their analysis allows $K$, $p$ and $N_i$ to grow with $n$. Unfortunately, the convergence rate of their estimator is not optimal  \citep{Tracy, Top}. The recent work of \cite{Tracy} is the first to  establish the minimax lower bound for the estimator of $A$ in topic models for known, fixed $K$. Their estimator provably achieves the minimax optimal rate under appropriate conditions. When $K$ is allowed to grow with $n$, the minimax optimal rate of $\|\wh A -A\|_1$ is established in \cite{Top}  and an optimal estimation procedure is   proposed. 
	
	Despite   these recent advances, all the aforementioned results are established for a \emph{fully dense} matrix $A$. In the modern big data era,   the dictionary size $p$, the number of documents $n$ and the number of topics $K$ are large, 
	as evidenced by  real data in Section \ref{sec_sim}.	
	Sparsity is likely to happen for large dictionaries ($p$)  and  when the number of topics $K$ is large, one should expect that there are many words \emph{not} occurring in all topics, that is, $A_{jk} = \PP(\text{word }j\ | \text{ topic }k)=0$ for some $k$. 
	
	To the best of our knowledge,  the minimax lower bound of  $\|\wh A - A\|_{1}$  in the topic model is unknown when the word-topic matrix $A$ is element-wise sparse and no estimation procedure exists tailored to this scenario of sparse $A$ and unknown $K$.

	\subsection{Our contributions}
	We summarize our contributions in this paper. 
	
	\paragraph{New minimax lower bound for  $\|\wh A - A\|_1$, when $A$ is sparse.} To   understand the difference of estimating a dense $A$ and a entry-wise sparse $A$ in topic models, we first establish the minimax lower bound of estimators of $A$ in Theorem \ref{thm_lb} of Section \ref{sec_lower_bound}.  It shows that 
	\[
		\inf_{\wh A} \sup_{A}\PP_{A}\left\{\|\wh A - A\|_{1} \ge c_0 \| A\|_1 \sqrt{ \| A\|_0  \over nN} \right\} \ge c_1.
	\]
	for some constants $c_0>0$ and $c_1\in (0, 1]$, by assuming $N= N_1 = N_2 = \cdots = N_n$ for ease of presentation. The infimum is taken over all estimators $\wh A$ while the supremum is over a prescribed parameter space $\A$ defined in (\ref{classA}) below. We have $\| A\|_1=K$ by (\ref{col_sum_one}) for all $A$.  The term  $\|A\|_0$   characterizes the overall sparsity of $A$, and  the minimax rate of $A$ becomes faster as $A$ gets more sparse.
	When the rows $A_{j\cdot}$ of non-anchor words $j$ are dense in the sense  $\| A_{j\cdot} \|_0 = K$, our result reduces to that in \cite{Top}. Our minimax lower bound is valid for all $p$, $K$, $N$ and $n$ and, to the best of our knowledge, the lower bound with dependency on the sparsity of $A$ is new in the topic model literature.
	
	\paragraph{A new estimation procedure for sparse $A$.} 
	To the best of our knowledge, the only minimax-optimal estimation procedure, for {\it dense} $A$ and $K$ large and unknown, is offered in \cite{Top}. While the procedure is computationally very fast, it is impractical to  adjust it in simple ways in order  to obtain a sparse estimator of $A$, that would hopefully be minimax-optimal.

	
	For instance, simply thresholding an estimator $\wh A$ to encourage sparsity will require threshold levels that vary from row to row, resulting in too many tuning parameters.
		We propose a new estimation procedure in Section \ref{sec_est_A} that adapts to this unknown sparsity. To motivate our procedure, we start with the recovery of  $A$ in the noise-free case in Section \ref{sec_noise_free_A},  under Assumptions \ref{ass_sep} and \ref{ass_pd_W}. 
	  Since   several existing algorithms, including \cite{Top},    provably select the anchor words, we mainly focus  on the estimation of the portion of $A$ corresponding to non-anchor words.

	In the presence of noise, we propose our estimator in Section \ref{sec_noise_est_A} and summarize the procedure in Algorithm \ref{alg_1}. The new algorithm requires the solution  of a quadratic program  for each non-anchor row. 
	Except for a ridge-type tuning parameter (which can often be set to zero), the procedure is devoid of any (further) tuning parameters.
	We give detailed comparisons with other methods in the topic model literature in Section \ref{sec_comp}.
	
	\paragraph{Adaption to sparsity.} We provide finite sample upper bounds on the $\ell_1$ loss of our new estimator in Section \ref{sec_upper_bound},   valid for all $p$, $K$, $n$ and $N$. As shown in Theorem \ref{thm_rate_Ahat}, our estimator   adapts to the unknown sparsity of $A$. To the best of our knowledge, our estimator is the first   computationally fast estimator shown  to   adapt to  the unknown sparsity of $A$.  We further show   in Corollary \ref{cor_opt_rate} that it is minimax optimal under reasonable scenarios.
	
	\paragraph{Simulation study.}
	In Section \ref{sec_sim}, we provide  experimental results based on both synthetic data and semi-synthetic data. 
	We compare our new estimator   with  existing state-of-the-art algorithms. The effect of sparsity on the estimation of $A$ is verified  in Section \ref{sec_sim_syn} for synthetic data, while 
	we analyze two semi-synthetic datasets based on a corpus of NIPs articles and a corpus of New York Times (NYT) articles in
	Section \ref{sec_sim_semi_syn}. 
	
	\subsection{Notation}
	We introduce notation that we   use throughout the  paper. 
		The integer set $\{1,\ldots,n\}$ is denoted by $[n]$. We use $\1_d$ to denote the $d$-dimensional vector with entries equal to $1$ and use $\{e_1, \ldots, e_K\}$  to denote the canonical basis vectors in $\R^K$. For a generic set $S$, we denote $|S|$ as its cardinality. For a generic vector $v\in \R^d$, we let $\|v\|_q$ denote the vector $\ell_q$ norm, for $q= 0,1,2,\ldots,\infty$, and let $\textrm{supp}(v)$ denote its support. We write $\|v\|_2 = \|v\|$ for brevity. We denote by $\diag(v)$ a $d\times d$ diagonal matrix with diagonal elements equal to $v$. For a generic matrix $Q\in \R^{d\times m}$, we write $\|Q\|_{{1}} = \sum_{1\le i\le d,1\le j\le m}|Q_{ij}|$ and $\|Q\|_{\r}=\max_{1\le i\le d}\sum_{1\le j\le m}|Q_{ij}|$.  For the submatrix of $Q$, we let $Q_{i\cdot}$ and $Q_{\cdot j}$ be the $i$th row and $j$th column of $Q$. For a set $S$, we let $Q_S$ and $Q_{\cdot S}$ denote its $|S|\times m$ and $m\times |S|$ submatrices. 
	For a symmetric matrix $Q$, we denote its smallest eigenvalue by $\lambda_{\min}(Q)$. We use $a_n \lesssim b_n$ to denote there exists an absolute constant $c>0$ such that $a_n \le cb_n$, and write $a_n \asymp b_n$ if there exists two absolute constants $c, c'>0$ such that 
	$c b_n \le a_n \le c'b_n$. In the probabilities of our results, we might write $c''a_n$ as $O(a_n)$ for some absolute constant $c''>0$. Finally, we write $a_n = o_p(b_n)$ if $a_n / b_n \to 0$ with probability tending to $1$.
	
	For a given word-topic matrix $A$, we let  $I := I(A)$ be the set of anchor words, and $\mathcal{I}$ be its partition relative to the $K$ topics. That is, \begin{eqnarray}\label{def_pure}
	I_k := \{j\in[p]: A_{jk}  >  0,  \ A_{j\ell} = 0\ \text{ for all } \ell \ne k\},\quad 
	I := \bigcup_{k=1}^{K} I_k,\quad  \I :=  \left\{I_1, \ldots, I_K \right\}.
	\end{eqnarray}
	We further write $J := [p]\setminus I$ to denote the set of non-anchor words. For the convenience of our analysis, we assume all documents have the same number of sampled words, that is, $N:= N_1 = \cdots = N_n$, while our results can be extended to the general case.

	\section{Minimax lower bounds of $\|\wh A - A\|_{1}$}  \label{sec_lower_bound}
	In this short section, we establish the minimax lower bound of $\|\wh A -A\|_{1}$ based on model (\ref{model}) for any estimator $\wh A$ of $A$ over the parameter space 
	\begin{align}\label{classA}
	\A  := \Bigl\{ A\in\R_+^{p\times K}:&\ A^\top \bm{1}_p = \bm{1}_K,~   \text{$A$ satisfies Assumption \ref{ass_sep} with $\| A\|_0 \le nN$}\Bigr\}.
	\end{align}
	To prove the lower bound, it suffices to choose one particular $W$. We let 
	\begin{equation}\label{def_W0}
	W^0 = \{\underbrace{e_1, \ldots, e_1}_{n_1}, \underbrace{e_2, \ldots, e_2}_{n_2}, \ldots, \underbrace{e_K, \ldots, e_K}_{n_K}\}
	\end{equation}
	with $\sum_{k =1}^{K}n_k = n$ and $|n_k-n_{k'}| \le 1$ for $k, k' \in [K]$. Note that $W^0$ satisfies Assumption \ref{ass_pd_W}.
	Denote by $\PP_{A}$ the joint distribution of $(X_1,\ldots,X_n)$ under model (\ref{model}), for the chosen $W^0$. 
	
	\begin{thm}\label{thm_lb}
		Under topic model (\ref{model}), assume (\ref{model_multinomial}).
		Then, there exist constants $c_0>0$ and $c_1\in (0, 1]$ such that
		\begin{align}\label{lower_bound}
			\inf_{\wh A} \sup_{A\in \A }\PP_{A}\left\{ \|\wh A - A\|_{1}  \ge c_0   \|A\|_1\sqrt{\|A\|_0 \over nN}  \right\} \ge c_1.
		\end{align}
		The infimum is taken over all estimators $\wh A$ of $A$. 
	\end{thm}
	

	\begin{remark}{\rm 
	The   estimate constructed in the next section achieves this lower bound in many scenarios.
		The lower bound rate   of $\|\wh A- A\|_1$ in (\ref{lower_bound}) becomes faster as $\| A\|_0$ decreases, that is, if $A$ becomes more sparse.
		Since each of the $K$ columns of $A$ sum to one, we always have $\| A\|_1=K$.
		If the submatrix $A_J$, corresponding to the non-anchor words, is dense in the sense that  $\|A_J\|_0 = K|J|$,  Theorem \ref{thm_lb}  reduces to the result in   \cite[Theorem 6]{Top} for  $K=K(n)$,  and the result in  \cite[Theorem 2.2]{Tracy} for fixed $K$. 
}	\end{remark}

	\section{Estimation of $A$}\label{sec_est_A}
	In this section, we present our procedure for estimating $A$ when a subset of anchor words $L=\bigcup_{k=1}^KL_k$ and its partition $\L = \{L_1, \ldots, L_K\}$ are given. Moreover, we assume that, for each $k\in [K]$, 
	$
	L_k \subseteq I_{\pi(k)}$
	for some group permutation $\pi: [K] \to [K]$. For simplicity of presentation, we assume $\pi$ is identity such that 
	\begin{equation}\label{def_L}
	L_k \subseteq I_{k}, \qquad \text{for each } k\in [K].
	\end{equation}
	We discuss methods for selecting $L$ and $\L$ in  Section \ref{sec_disc_L}. We start with the noise-free case, that is,  we  observe the expected word-document frequency matrix $\M$,  in Section \ref{sec_noise_free_A}. Motivated by the developed algorithm in the noise-free case that recovers $A$, we propose the estimation procedure of $A$ in Section \ref{sec_noise_est_A} when we   have access to $X$ only.
	
	\subsection{Recovery of $A$ in the noise-free case} \label{sec_noise_free_A}
	Suppose that $\M$ is given and write 
	 $D_\M := n^{-1}\diag(\M \1_n)$ and $D_W := n^{-1}\diag(W\1_n)$.  We recover $A$ via its row-wisely normalized version \begin{equation}\label{def_B}
 	B=D_\M^{-1}AD_W
	 \end{equation}
	 as $B$ enjoys the following three properties:
	\begin{equation}\label{prop_B}
		\textrm{supp}(B) = \textrm{supp}(A),\quad B_{jk}\in [0,1],\quad \|B_{j\cdot}\|_1 = 1, \qquad\text{for all }j\in[p], k\in[K]. 
	\end{equation}
The row-wise sum-to-one property is critical in the later estimation step to adapt to the unknown sparsity of $B$ (or equivalently, the sparsity of $A$).
		From $\L= \{L_1, \ldots, L_K\}$ and (\ref{prop_B}), we can directly recover $B_L$ by setting 
	\[
		B_{i\cdot} = e_k, \qquad \text{for any }i\in L_k, k\in [K].
	\]
	To recover $B_{L^c}$ with $L^c:=[p]\setminus L$, let 
	\[
	R := D_\M^{-1}\Theta D_\M^{-1}  = 
	B\left(D_W^{-1}{1\over n}W W^\top D_W^{-1}\right)B^\top := BMB^\top 
	\]
  be a normalized version of $$\Theta := n^{-1}\M\M^\top .$$ 
  Since $R$ has the  decomposition 
	\[
		R_{LL} = B_LMB_L^\top ,\qquad R_{L^cL} = B_{L^c}MB_L^\top .
	\]
	and Assumption \ref{ass_pd_W} implies $M$ is invertible,  
	we arrive at the expressions 
	\begin{eqnarray}
		M &=&
			(B_L^\top B_L)^{-1}B_L^\top R_{LL}B_L(B_L^\top B_L)^{-1}, \label{een}\\
			B_{L^c} &=& R_{L^cL}B_L(B_L^\top B_L)^{-1}M^{-1}. \label{twee}
	\end{eqnarray}
	Display (\ref{twee}) implies that $$M B_{L^c}^\top  = ( B_L^\top  B_L)^{-1} B_L^\top  R_{LL^c} :=H,$$ whence 
	$M\beta =h$ for each column $\beta$ of $B_{L^c}^\top $ (which is a  {\em row} of $B_{L^c}$) and corresponding column $h$ of $H$. Given $M$ and $H$, the solution $\beta$ of the equation $M\beta=h$ is  the minimizer of $\beta^\top  M \beta - 2\beta^\top  h$ over $\beta\ge 0$ and $\|\beta\|_1=1$.
	This formulation will be used in the next subsection.
	
	After recovering $B^\top  = (B_L^\top , B_{L^c}^\top )$, display (\ref{def_B}) implies that $A$ can be recovered by normalizing columns of $D_\M B$ to unit sums.

	\subsection{Estimation of $A$ in the   noisy case}\label{sec_noise_est_A}
	The estimation procedure of $A$ follows the same idea of the noise-free case. 
		We first estimate $B$ defined in (\ref{def_B}) by using the estimate  
	\begin{equation}\label{def_R_hat}
	\wh R = D_X^{-1} 	\wh\Theta  D_X^{-1}
	\end{equation}
	of $R$, based on $D_X = n^{-1}\diag(X\1_n)$ and the unbiased estimator 
	\begin{equation}\label{est_Theta}
	\wh\Theta = {1\over n}\sum_{i =1}^n\left[
	{N_i \over N_i - 1}X_iX_i^\top  - {1\over N_i-1} \textrm{diag}(X_i)\right]
	\end{equation}
of the matrix $\Theta$.
We estimate $B_L$ by 
	\begin{equation}\label{est_BI}
	  \wh B_{i\cdot } = e_k, \quad \text{for any }i\in L_k, k\in [K].
	\end{equation}
	Based on 	\begin{equation}\label{est_M_h}
		\wh M =  (\wh B_L^\top \wh B_L)^{-1} \wh B_L^\top \wh R_{LL} \wh B_L(\wh B_L^\top \wh B_L)^{-1},\qquad 
		\wh H =  (\wh B_L^\top \wh B_L)^{-1} \wh B_L^\top  \wh R_{LL^c}, 
	\end{equation}
	we estimate row-by-row the remainder of the matrix $B$. We compute,   for each $j\in L^c$,  
	\begin{alignat}{2}\label{est_BJ_1}
	\wh B_{j\cdot } &= 0, &&\quad \text{if }(D_X)_{jj}\le  7\log(n\vee p) / (nN),\\\label{est_BJ_2}
	\wh B_{j\cdot } &=\arg\min_{\beta \ge 0,\ \|\beta\|_1 = 1}\beta^\top  (\wh M  + \lambda \bI_K)\beta - 2\beta^\top  \wh h^{(j)}, &&\quad \text{otherwise,}
	\end{alignat}
	where $\wh h^{(j)}$ is the corresponding column of $\wh H$.
	We set $\lambda = 0$ whenever $\wh M$ is invertible and otherwise choose $\lambda$  large enough
	such that $\wh M + \lambda \bI_K$ is invertible. We detail the exact rate of $\lambda$ when $\wh M$ is not invertible in Section \ref{sec_upper_bound}.
	Finally, we  estimate $A$ via normalizing $D_X \wh B$ to unit column sums.
	
	\begin{remark}{\rm 
		In our procedure, the hard-thresholding step in (\ref{est_BJ_1}) is critical to obtain the optimal rate of the final estimator that does not rely    on a lower bound condition on the word-frequencies. In contrast,  the analysis of \cite{arora2013practical} requires a lower bound for all word-frequencies. The thresholding level in (\ref{est_BJ_1}) is carefully chosen  from the element-wise control of the difference $D_X- D_\M$.
}	\end{remark}
	
	For the reader's convenience, the estimation procedure is summarized in Algorithm \ref{alg_1}.  
	\begin{algorithm}[H]
		\caption{Sparse Topic Model solver (STM)
		}\label{alg_1}
		\begin{algorithmic}[1]
			\Require frequency data matrix $X\in\R^{p\times n}$ with document lengths $N_1, \ldots, N_n$;  the partition of anchor words $\L$
			\Procedure{}{}
			\State compute $D_X = n^{-1}\diag(X\1_n)$, $\wh \Theta$ from (\ref{est_Theta}) and $\wh R$ from (\ref{def_R_hat})
			\State compute $\wh B_L$ from (\ref{est_BI})
			\State compute $\wh M$ and $\wh H$ from (\ref{est_M_h})
			\State solve $\wh B_{L^c}$ from (\ref{est_BJ_1}) -- (\ref{est_BJ_2}) by using $\lambda$ in (\ref{rate_lambda_data})
			\State compute $\wh A$ by normalizing $D_X\wh B$ to unit column sums
			\State \Return $\wh A$
			\EndProcedure
		\end{algorithmic}
	\end{algorithm}
	
	\subsection{Comparison with existing methods}\label{sec_comp}
	In this section, we provide   comparisons between our estimation procedure and two existing methods, which are seemingly close to our procedure.
	
	 \paragraph{Comparison with    \cite{arora2013practical}.} 
	 This algorithm also  estimates the same target $B$ defined in (\ref{def_B}) first. For a given set $L$ of anchor words, there are two main differences for estimating $B$.  
	 \begin{enumerate}
	     \item The algorithm in \cite{arora2013practical}  uses \emph{only one} anchor word per topic to estimate $B$ whereas our estimation procedure utilizes all anchor words. The benefit of using multiple anchor words per topic is substantial and verified in our simulation in Section \ref{sec_sim}.
	     \item  The algorithm in \cite{arora2013practical} is based on different quadratic programs with more parameters ($pK$ versus $K^2$).
	     This makes it more computationally intensive and less accurate than the algorithm proposed here.
	     This is verified in our simulations in Section \ref{sec_sim_semi_syn}.
  { Specifically, 
	 write
	 $\wt \Theta := D_{\Theta}^{-1}\Theta= D_{\Theta}^{-1}  A (n^{-1}WW^\top  ) A^\top  $, $Q:= (n^{-1}WW^\top )A^\top $ and $\wt Q := D_Q^{-1}Q$ with  $D_{\Theta} = \diag(\Theta \1_p)$ and $D_Q = \diag(Q\1_p)$. \cite{arora2013practical} utilizes the following observation
	 \[
	 \wt\Theta = D_{\Theta}^{-1}A Q = D_{\Theta}^{-1}AD_Q \wt Q = B \wt Q
	 \]
	 by noting that 
	 $
	 D_{\Theta} = D_\M $ and $D_Q = D_W 
	 $
	 from (\ref{col_sum_one}). 
	 Based on the observation that 
	 $\wt \Theta_{j\cdot}\in \R^p$ is a convex combination of $\wt \Theta_{\wt L} = \wt Q\in \R^{K\times p}$ for any $j\in [p]\setminus \wt L$, 
	 \cite{arora2013practical} proposes to estimate $B_{j\cdot}$ by solving 
	 \begin{equation}\label{awr_est_B}
	 \wh B_{j\cdot} = \arg\min_{\beta\ge 0, \|\beta\|_1 = 1} \left\|\wh {\wt \Theta}_{j\cdot} - \beta^\top \wh {\wt Q}\right\|^2
	 \end{equation}
	 where $\wh{\wt \Theta}_{j\cdot}$ and $\wh{\wt Q}$ are the corresponding estimates of $\wt \Theta_{j\cdot}$ and $\wt Q$.
	 The matrix $\wt Q$ contains $p\times K$  entries, while    our estimation procedure in (\ref{est_BJ_2}) only requires to estimate $M\in \R^{K\times K}$ which has fewer parameters. The analysis of \cite{arora2013practical} only holds for invertible   estimates  $\wt Q\wt Q^\top $ and the rate of the estimator from (\ref{awr_est_B}) depends on $\lambda_{\min}(\wt Q\wt Q^\top )$. Our result holds as long as $\lambda_{\min}(M)>0$ due to the ridge-type estimator in (\ref{est_BJ_2}) and the rate of our estimator in (\ref{est_BJ_2}) depends on $\lambda_{\min}(M)$. Lemma \ref{lem_QQ} in the Appendix shows that  
	 \[
	  \lambda_{\min}(M) \lambda_{\min}(n^{-1}WW^\top ) \min_{k\in [K], i\in I_k}A_{ik}^2
	 \le \lambda_{\min}(\wt Q\wt Q^\top ) \le \lambda_{\min}(M).
	 \]
	 Since $0<\lambda_{\min}(n^{-1}WW^\top ) \le 1/K$ as shown in Lemma \ref{lem_lbd_min_C} and $0<\min_{i\in I_k, k\in [K]}A_{ik}^2 < 1$, it is easy to see that $\lambda_{\min}(\wt Q\wt Q^\top )$ could be much smaller comparing to $\lambda_{\min}(M)$. This suggests that our procedure in (\ref{est_BJ_2}) should be more accurate than (\ref{awr_est_B}), which is confirmed in our simulations in Section \ref{sec_sim_semi_syn}.} 
	
	 \end{enumerate} 
	 
	 \paragraph{Comparison  with \cite{Top}.}
	 Although both methods are based on the normalized second moment $R$, they differ significantly in estimating $A$. 
	 \begin{enumerate}
	     \item 
	 The algorithm in \cite{Top}   uses $R$ {\it only} to estimate the anchor words and relies on $\Theta$ for the estimation of $B$. Specifically, by observing
	\[
	\Theta_{\cdot \wt L} := ACA_{\wt L} = AA_{\wt L}^{-1}A_{\wt L}CA_{\wt L} = AA_{\wt L}^{-1} \Theta_{\wt L\wt L} := \wt A \Theta_{\wt L\wt L}
	\]
	with $\wt L$ being a set that contains one anchor word per topic and $A_{\wt L}\in \R^{K\times K}$ being a diagonal matrix, \cite{Top} proposes to first estimate 
	$
	\wt A
	$ 
	by $\wh \Theta_{\cdot \wt L}\wh \Omega $. Here 
	$\wh \Omega$ is an estimator of $\Theta_{\wt L\wt L}^{-1}$ obtained  via solving a linear program.
	Instead of $\wt A$, we propose here to first estimate  $B$ defined in (\ref{def_B}).
	This is a different scaled version of $A$ with more desirable structures  (\ref{prop_B}). 
	\item
	Furthermore,
	our estimation of $B$ is done row-by-row via quadratic programming instead of simple matrix multiplication. While this is more computationally expensive than estimating $\Theta_{\wt L\wt L}^{-1}$, it gives more accurate row-wise control of $\wh B - B$.
	This control is the key to obtain a faster rate of $\|\wh A -A\|_1$ that adapts to the unknown sparsity. 
	\item
	Finally, we emphasize that it is impractical to modify the estimator of \cite{Top} to adapt to the sparsity of $A$. For instance,   further thresholding the estimator of $\wt A$ to encourage sparsity, will require the thresholding levels to vary row-by-row. This would  involve too many tuning parameters. \end{enumerate}

	\section{Upper bounds of $\|\wh A -A\|_1$}\label{sec_upper_bound}
	To simplify notation and properly adjust the scales, for each $j\in[p]$ and  $k\in[K]$, we define
	\begin{equation}\label{def_alpha_gamma}
	\u_j := {p\over n}\sum_{i =1}^n \M_{ji},\qquad {\g_k} := {K\over n}\sum_{i=1}^nW_{ki},\qquad \alpha_j := p\max_{1\le k\le K} A_{jk}, 
	\end{equation}
	such that $\sum_{j=1}^p \u_j = p$, $\sum_{k=1}^K\g_k = K$ and $p\le \sum_{j=1}^p \alpha_j \le pK$ from (\ref{col_sum_one}). For given set $L$ satisfying (\ref{def_L}), we further set 
	\begin{equation}\label{def_oaua}
	 \uu_L =  \min_{i\in L}\u_i, \quad \og = \max_{1\le k\le K} \g_k,\quad
	 \ug = \min_{1\le k\le K} \g_k,\quad \ua_L =\min_{i\in L}\alpha_i,\quad \rho_j = \alpha_j / \ua_L.
	\end{equation}
	For future reference, we note that 
	\[ 
		\og \geq 1\ge \ug.
	\] 
	As our procedure depends whether the inverse of $\wh M$ defined in (\ref{est_M_h}) exists, we first give a critical bound on the control for the operator norm of  $\wh M - M$ and provide insight on the choice of  $\lambda$ in (\ref{est_BJ_2}).

	\begin{lemma}\label{lem_M_hat}
		Consider the topic model (\ref{model}) under  assumption \ref{ass_sep}  and
		\begin{equation}\label{cond_Pi_I}
			\min_{i\in L}{1\over n}\sum_{i =1}^n \M_{ji} \ge  {c_0\log(n\vee p)\over N},\qquad 	\min_{i\in L}\max_{1\le i\le n}\M_{ji} \ge  {c_1\log^2(n\vee p)\over N}
		\end{equation}
		for some sufficiently large constants $c_0, c_1>0$. 
		Then, with probability $1-O((n\vee p)^{-1})$, we have 
		\begin{equation}\label{rate_M_op}
			\|\wh M - M\|_{\rm{\rm{\rm{op}}}} \lesssim {K \over \ug}\sqrt{pK\log(n\vee p) \over \uu_LnN}.
		\end{equation}
	\end{lemma}
	
	\begin{remark}
	{\rm 	 \cite{arora2013practical} observe  that the smallest frequency of anchor words plays an important role in the estimation of $A$.  Condition (\ref{cond_Pi_I}) prevents the frequency of anchor words from being too small and   also appeared in  \cite{Top}.
}	\end{remark}
	
	In case the matrix  $\wh M$ cannot be inverted, we select $\lambda \ge \|\wh M - M\|_{\rm{\rm{op}}}$ in (\ref{est_BJ_2}).  Lemma \ref{lem_M_hat} thus suggests to choose $\lambda$ as 
	 \begin{equation}\label{rate_lambda_thm}
	 \lambda = c\cdot {K\over \ug }\sqrt{pK\log(n\vee p) \over \uu_L nN},
	 \end{equation}
	 for some absolute constant $c>0$. 	Let $\wh A$ be obtained via choosing $\lambda$ as (\ref{rate_lambda_thm}). The following theorem states the upper bound of $\|\wh A -A\|_1$.  Our procedure, its theoretical performance  and its proof differ from those in \cite{Top}.   While the proof borrows some preliminary lemmas from \cite{Top}, it requires a 
	 more refined analysis (see Lemmas \ref{lem_h_hat} -- \ref{lem_Rem_13} in the Appendix).\\

	 \noindent 
	We define 
	$s_j = \|A_{j\cdot}\|_0$ for $j\in [p]$, $s_J = \sum_{j \in J}s_j$ and  $\wt s_J:= \sum_{j\in L^c}( \alpha_j / \ua_L ) s_j= \sum_{j\in L^c} \rho_js_j$.

	\begin{thm}\label{thm_rate_Ahat} 
		Under model (\ref{model}), assume  Assumptions \ref{ass_sep}, \ref{ass_pd_W}  
		with $\lambda_{\min} := \lambda_{K}(n^{-1} WW^\top )>0$ and 
	  (\ref{cond_Pi_I}).  
		Then, with probability $1-O((n\vee p)^{-1})$, we have
		\begin{align*}
		\|\wh A - A\|_1 \lesssim {\rm I + II + III},
		\end{align*}
		where 
		\begin{align*}
	{\rm	I } &=  {K\over \ug}\sqrt{p\log(n\vee p) \over nN} + {pK\log(n\vee p) \over \ug nN} \\
	{\rm 	II} & =   {\og^2 \over \ug K\lambda_{\min}} \Biggl\{
		\max\left\{s_J+|I|-|L|, \wt s_J\right\}\left({K\log(n\vee p)\over \ug n N} + \sqrt{p\log^4(n\vee p) \over \uu_LnN^3}\right)\\
		& ~\quad\qquad \qquad + K\sqrt{\max\left\{
		s_J+|I|-|L|, \wt s_J\right\}{\log(n\vee p) \over \ug nN}}
		\Biggr\}\\
		{\rm III }&= K\sqrt{K\wt s_J \cdot {\og \over \ug}\cdot {\log(n\vee p) \over \ug nN}}
		\end{align*}
		Furthermore, if 
		\begin{equation}\label{cond_C}
			\lambda_{\min} \ge c_2{\og^2 \over \ug}\sqrt{p\log(n\vee p)\over \uu_L KnN}
		\end{equation}
		for some sufficiently large constant $c_2>0$, then, with probability $1-O((n\vee p)^{-1})$, $\wh A$ obtained via $\lambda =0$ enjoys the same rate with ${\rm III} = 0$.
	\end{thm}
	\begin{remark}{ \rm 
	The estimation error of $A$ consists of three parts: I, II and III. Each part reflects 
	errors made at different stages of our estimation procedure.
	Recall that $\wh A$ first uses a hard-thresholding step in (\ref{est_BJ_1}) and then relies on the estimates $\wh B$ and $D_X$ of $B$ and $D_\M$, respectively. The first term in ${\rm I}$ quantifies the error of $D_X - D_\M$, while the second term is due to the hard-thresholding step. The second term ${\rm II}$ is due to the error of $\wh B_j - B_j$ for those $j \in [p]\setminus L$ that pass  the test  (\ref{est_BJ_1}). Finally, ${\rm III}$ stems from the error incurred by  the regularization choice of $\lambda$.  
}	\end{remark}

	\begin{remark}{ \rm 
	Condition (\ref{cond_C}) is  a lower bound for the smallest eigenvalue of the matrix $  n^{-1}WW^\top $. If it holds, we can set $\lambda = 0$ with high probability and the rate of $\|\wh A - A\|_1$ is improved by (at most) a factor of $\sqrt{K(\og / \ug)}$. Under (\ref{cond_Pi_I}), inequality (\ref{cond_C}) follows from
			\[
				\lambda_{\min} \ge {c_2\over \sqrt{c_0}} {\og^2 \over \ug}\sqrt{1\over  Kn}.
  			\]
}	\end{remark}
\medskip
	
	The following corollary provides sufficient conditions that guarantee that our estimator $\wh A$ constructed in Section \ref{sec_noise_est_A} achieves the optimal minimax rate.

		\begin{cor}[Attaining the minimax rate]\label{cor_opt_rate}
	Consider the topic model  (\ref{model}) with  Assumptions \ref{ass_sep} and  \ref{ass_pd_W}.  Suppose further that
			\begin{enumerate}
				\item[(i)] $\| A\|_0\log(n\vee p)\lesssim nN$;
				\item[(ii)]  $\og \asymp \ug$,\quad  $\lambda_{\min}\asymp 1/K$;
				\item[(iii)] $\sum_{j \in L^c}\rho_j s_j \lesssim s_J+|I|$
			\end{enumerate}
			hold. Further, assume (\ref{cond_Pi_I}) holds with the condition on $\min_{i\in L}n^{-1}\sum_{i = 1}^n\M_{ji}$ replaced by
			\begin{equation}\label{cond_Pi_I_new}
			\min_{i\in L}{1\over n}\sum_{i =1}^n \M_{ji} \ge  c_0\max\left\{1, {(s_J+|I|-|L|)\log^2(n\vee p) \over K^2N}\right\}{\log(n\vee p)\over N}.
			\end{equation}
			Then, with probability $1-O((n\vee p)^{-1})$, we have
		\[
			\|\wh A - A\|_1 
			\lesssim 
			\| A\|_1 \sqrt{\| A\|_0\log(n\vee p) \over  nN}.
		\]
	\end{cor}
	\begin{remark}[Conditions in Corollary \ref{cor_opt_rate}]\mbox{}{\rm
	\begin{enumerate}
		\item  Condition $(i)$ is natural (up to the multiplicative  logarithmic factor) as $\| A\|_0$ is the effective number of parameters to estimate while $nN$ is the total sample size. 
		\item  The first part of condition $(ii)$, $\ug \asymp \og$, requires that all topics have similar frequency. The ratio $\og / \ug$ is called the \emph{topic imbalance}  \citep{arora2012learning} and is expected to affect the estimation rate of $A$.
		\item
		The second part
		of condition $(ii)$, $\lambda_{\min} \asymp 1/K$, requires that topics are not too correlated. This is expected even for known $W$, playing the same role of the design matrix in the classical regression setting. 
		\item Condition $(iii)$ puts a mild constraint on the word-topic matrix $A$ between the selected anchor words and the other words
		(anchor and non-anchor). It is implied by
		\[
		\sum_{j \in L^c} {s_j \over \sum_{j \in L^c}s_j}\|A_{j\cdot}\|_\i  \lesssim \min_{i\in L}\|A_{i\cdot}\|_1,
		\]
		which in turn is implied by
		\[
		\max_{1 \le k\le K}\PP\left\{\text{word }j \ |\ \text{topic }k\right\}~ \lesssim~ \sum_{k = 1}^K\PP\left\{
		\text{word }i \ |\ \text{topic k}
		\right\}
		\]
		for any $i\in L$ and $j \notin L$. The latter condition  prevents the selected anchor words from being much less  frequent than the other words. 
		
		\item 
		Finally, condition (\ref{cond_Pi_I_new}) strengthens (\ref{cond_Pi_I}) by requiring a slightly larger lower bound for the frequency of selected anchor words. It is implied by 
		\[
		N \ge {\|A\|_0 \log^2(n\vee p)  \over K^2} \ge {(s_J + |I| - |L|) \log^2(n\vee p) \over K^2}
		\]
		under (\ref{cond_Pi_I}). As discussed in \cite{arora2012learning, arora2013practical, Top}, usage of  infrequent anchor words often leads to inaccurate estimation of $A$.
	\end{enumerate}
	}

	\end{remark}


	\section{Practical aspects of the algorithm}\label{sec_disc_L}
	We  discuss two practical concerns of our proposed algorithm in Section \ref{sec_noise_est_A}:
	\begin{enumerate}
	    \item Selection of the number of topics $K$ and subset of anchor words $L$
	    \item Data-driven choice of the tuning parameter $\lambda$ in (\ref{rate_lambda_thm}).
	\end{enumerate}
	
	\paragraph{Selection of $K$ and $L$.}
	Several existing algorithms  with theoretical guarantees for finding anchor words in the topic model exist.
	Most methods rely on finding the vertices of a simplex structure,
	{\em provided that the number of topics $K$ is known beforehand}.
		For known $K$, \cite{rechetNMF} make clever use of the appropriately defined simplex structure  on $\Theta=n^{-1} \Pi \Pi^\top $ implied by Assumption \ref{ass_sep}.
		%
		However, their method needs to solve a linear program in dimension  $p\times p$, which 
		becomes rapidly computationally intractable.   \cite{arora2013practical} proposes a faster combinatorial  algorithm which returns  one anchor word per topic. The returned anchor words are shown to be \emph{close to} anchor words within a specified tolerance level. Recently, \cite{Tracy} proposes another algorithm for finding anchor words by utilizing the simplex structure of the singular vectors of the word-document frequency matrix.  However, their algorithm runs much slower than that of \cite{arora2013practical}. 
		
	In practice, $K$ is rarely known in advance. This situation is addressed in \cite{Top}. This work proposes a  method that provably estimates $K$ from the data, provided that the topic-document matrix $W$ satisfies the following incoherence condition.
	\begin{ass}\label{ass_w}
		The inequality 
		\[ \cos\left( \angle( W_{i\cdot},   W_{j\cdot} ) \right)< \frac{\zeta_i}{\zeta_j} \wedge  \frac{\zeta_j}{\zeta_i}\qquad \text{for all $1\le i \ne j\le K$},\]
		holds, with $\zeta_i:= \| W_{i\cdot} \|_2 / \|W_{i\cdot} \|_1$.
	\end{ass}
	This additional assumption is not needed in the aforementioned work when $K$ is known. When columns of $W$ are i.i.d. samples of Dirichlet distribution, Assumption \ref{ass_w} holds with high probability under mild conditions on the hyper-parameter of Dirichlet distribution \cite[Lemma 25 in the Supplement]{Top}.
	  In addition to the estimation of $K$, the algorithm in \cite{Top} estimates both the set and the partition of  \emph{all} anchor words for each topic. This sets it further apart from 
	  \cite{arora2013practical}, as the latter only recovers \emph{one} approximate anchor word for each topic. The algorithm of finding anchor words in \cite{Top} is optimization-free and runs as fast as that in \cite{arora2013practical}.
	
	Hence, for selecting $L$, we can use Algorithm 4 in \cite{arora2013practical} when $K$ is known and   Algorithm 2 in \cite{Top} if $K$ is known or needs to be estimated. 

	\paragraph{Data-driven choice of $\lambda$.}  The precise rate for $\lambda$ in (\ref{rate_lambda_thm}) contains unknown  quantities $\bar{\gamma}$,  $\ug$ and $\uu_L$.
	We proceed via  cross-validation over a specified grid. 
	We prove in Lemma \ref{lem_mu_hat} in the Appendix that $|\min_{i\in L}(D_X)_{ii} - \uu_L/p| = o_p(\sqrt{\log(n\vee p) / (nN)})$ with $D_X = n^{-1}\diag(X\1_n)$. We recommend the following procedure for selecting $\lambda$. For some constant $c_0$ (our empirical study suggests the choice $c_0 = 0.01$),
	we take 
	\[
		t^* = \arg\min\left\{t\in \{0,1,2,\ldots\}:\, \wh M + \lambda(t) \text{ is invertible}\right\},
	\]
	with
	 \begin{equation}\label{rate_lambda_data}
	 \lambda(t) = t \cdot c_0\cdot K\left({K\log (n\vee p) \over [\min_{i\in L}(D_{X})_{ii}]n}\cdot  \frac{1}{n}\sum_{i=1}^n {1\over N_i}\right)^{1/2}.\\
	 \end{equation}
	
	\section{Experimental results}\label{sec_sim}
	
	In this section, we report on the empirical performance of the new algorithm proposed and compare it with existing competitors on both synthetic and semi-synthetic data. 
	
	\paragraph{Notation.} 	Recall that $n$ denotes the number of documents, $N$ denotes the number of words drawn from each document, $p$ denotes the dictionary size, $K$ denotes the number of topics, and $|I_k|$ denotes the cardinality of anchor words for topic $k$. We write $\xi:= \min_{k\in [K],i\in I_k}A_{ik}$ for the minimal frequency of anchor words. 
	Larger values of $\xi$ are more favorable for estimation. 
	
	\paragraph{Methodology.}
	For competing algorithms, we consider Latent Dirichlet Allocation (LDA)  \citep{BleiLDA}\footnotemark\footnotetext{We use the code of LDA from \cite{lda} implemented via the fast collapsed Gibbs sampling with the default of 1,000 iterations}, the algorithm (AWR) proposed in \cite{arora2013practical} and the TOP algorithm proposed in \cite{Top}. We use the default values of hyper-parameters for all algorithms. Both LDA and AWR need to specify the number of topics $K$.  In our proposed Algorithm \ref{alg_1} (STM), we choose $\lambda$ according to (\ref{rate_lambda_data}) and we select the anchor words either via AWR with specified $K$ or via TOP \citep{Top}, and proceed with the estimation of $A$ as described in Section \ref{sec_noise_est_A}. We name the resulting estimates
	 STM-AWR and STM-TOP, respectively.
	
	\subsection{Synthetic data}\label{sec_sim_syn}
	In this section, we use synthetic data to demonstrate the effect of the sparsity of $A$ on the estimation error $\|\wh A-A\|_1/K$ for AWR, TOP, STM-AWR and STM-TOP. Both AWR and STM-AWR are given the correct $K$, while TOP and STM-TOP estimate $K$.

	To simulate synthetic data, we  generate  $A$ satisfying Assumption \ref{ass_sep} by the following strategy. 
	\begin{itemize}[topsep = 8pt]\setlength{\itemsep}{1pt}
	    \item Generate  anchor words by $A_{ik}:=\xi$ for any $i\in I_k$ and $k\in [K]$.
	    \item Each entry of non-anchor words is sampled from   Uniform$(0,1)$.
	    \item Normalize each sub-column  $A_{Jk}\subset A_{\cdot k}$ to have sum  $1- \sum_{i\in I}A_{ik}$. 
	    \item Draw  columns of $W$   from the symmetric Dirichlet distribution with parameter $0.3$. 
	    \item Simulate $N$ words   from $\text{Multinomial}_p(N; AW)$. 
	\end{itemize} 
	
	\noindent
	To change the sparsity of $A$,  we randomly set $s= \floor{\eta   K}$ entries of each row in $A_J$ to zero, for a given sparsity proportion $\eta \in (0, 1)$.
	Normalizing the thresholded matrix gives $A(\eta)$ and the sparsity of $A(\eta)$ is calculated as  $s(\eta) = \|A(\eta)\|_0 / (pK)$. 
		We set 
		\begin{itemize}
		\centering
		\item[] 
		$N = 1500,\ p = n = 1000,\ K = 20,\ |I_k| = p / 200$ and $\xi =  K / p$. 	\end{itemize} 
		\noindent
	For each $\eta \in \{0, 0.1, 0.2, \ldots, 0.9\}$, we generate $50$ datasets based on $A(\eta)$ and report in Figure \ref{fig_syn} the average estimation errors $\|\wh A-A(\eta)\|_1/K$ of the four different algorithms.  The x-axis represents the corresponding sparsity level $s(\eta)$. Since the selected anchor words are up to a group permutation, we align the columns of $\wh A$ before calculating the estimation error.\\
	
	\noindent
	{\bf Conclusion.} STM-TOP has the best performance overall. Both STM-AWR and STM-TOP perform increasingly better as $A$ becomes  sparser.  The performance of AWR    improves only if the sparsity level is sufficiently large, say $s(\eta) < 0.5$. As expected, TOP does not adapt to the sparsity. 
	
	\begin{figure}[ht]
		\centering
		\includegraphics[width =.4\textwidth]{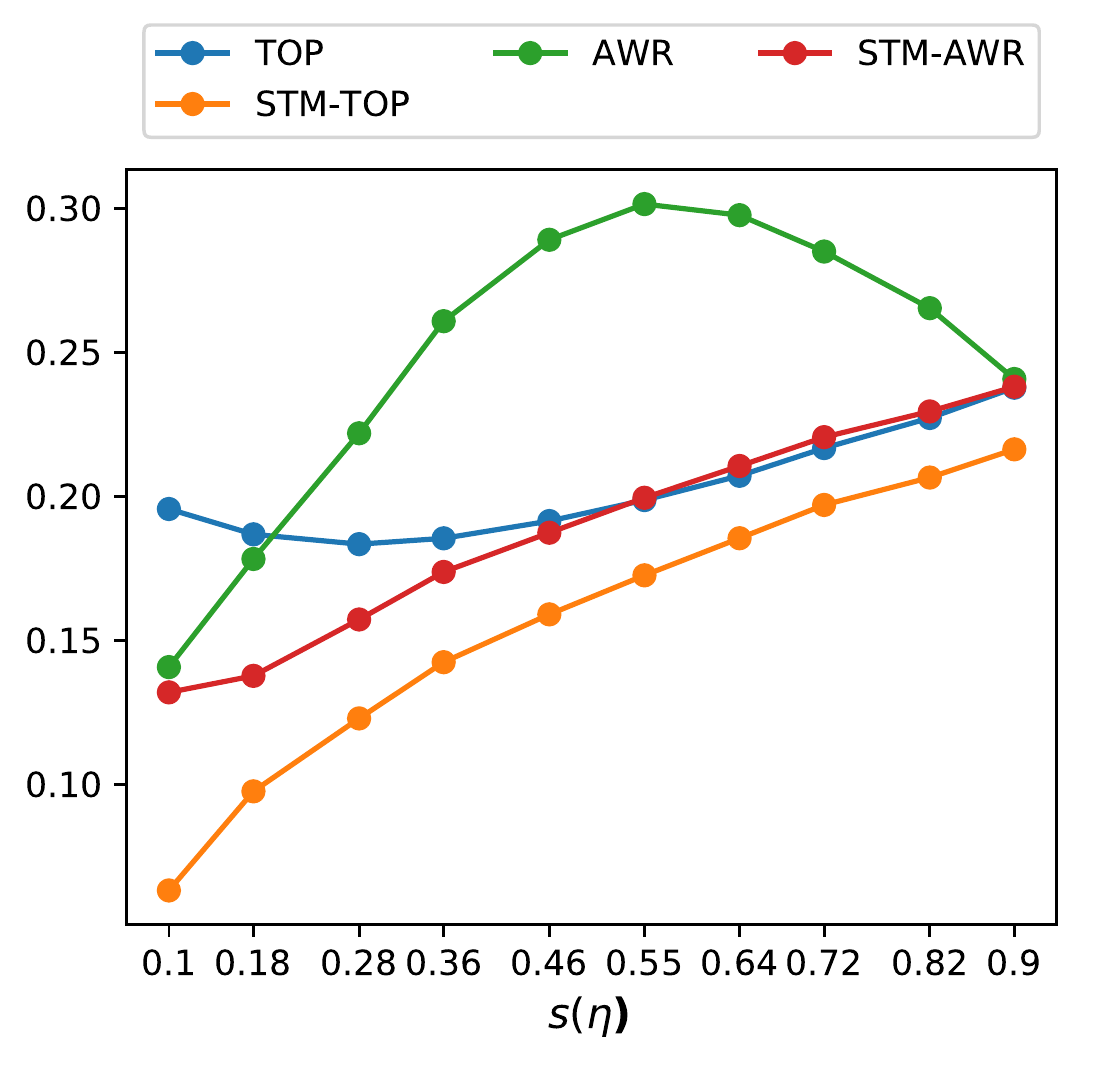}
		\caption{Plots of the estimation error  $\|\wh A -A(\eta)\|_1/K$ for $\eta\in \{0, 0.1, 0.2, \ldots, 0.9\}$. }
		\label{fig_syn}
	\end{figure}
	
	\subsection{Semi-synthetic data}\label{sec_sim_semi_syn}
	
	We evaluate 
	two real-world datasets, a corpus of NIPs articles and a corpus of New York Times (NYT) articles \citep{ucldata}. Following \citep{arora2013practical}, 
	\begin{enumerate}[topsep = 3mm, itemsep = 0pt]
	    \item We removed common stopping words and rare words occurring in less than
	150 documents.
	\item For each preprocessed dataset, we  apply LDA with $K = 100$ and obtain an estimated word-topic matrix $A^{(0)}$. 
	\item For each document $i\in [n]$, we generate the topics $W_i$ from a specified distribution. 
	\item We sample $N$ words from Multinomial$_p(N; A^{(0)}W)$.  
	\end{enumerate}

	\subsubsection{NIPs corpus}
	After this preprocessing stop, the NIPs dataset consists of  
	 $n = 1,500$ documents with dictionary size $p = 1,253$ and mean document length $847$.  
	 \begin{enumerate}
	     \item  
	We set $N = 850$ and vary $n\in \{2000, 4000, 6000, 8000, 10000\}$ for generating semi-synthetic data.
	\item 
 While  the estimated $A^{(0)}$ from LDA does not have exact zero entries, we calculate \emph{the approximate sparsity level} of $A$ by 
	\begin{equation}\label{s_A}
	\textbf{sparsity} = {1\over pK}\sum_{j=1}^p\sum_{k=1}^K 1 {\left\{A_{jk} \ge 10^{-3} p^{-1} \right\}} \approx 0.696.
	\end{equation}
	 The calculated \textbf{sparsity} indicates that the posterior $A^{(0)}$ from LDA has many entries close to $0$. 
	 \item 
	As in \cite{arora2013practical}, we manually add $|I_k| = m$ anchor words for each topic with $m \in \{1, 5\}$. After adding $m$ anchor words, we re-normalize the columns to obtain $A^{(m)}$.
	\item 
	The columns of $W$ are generated from the symmetric Dirichlet distribution with parameter 0.03. We sample $N$ words from Multinomial$_p(N; A^{(m)}W)$. 
	 \end{enumerate}
	For each combination of $n$ and $m$, we generate $20$ datasets and
	the average  estimation errors $\|\wh A - A\|_1/K$ of  different algorithms are shown in Figure \ref{fig_nips_diri}. The bars represent the standard deviations across 20 repetitions. Again, LDA, AWR and STM-AWR are given the correct $K$, while TOP and STM-TOP estimate $K$.\\
	
	\noindent 
	{\bf Conclusion.}  
	STM-TOP has  best overall performance and STM-AWR has the second best result. LDA is dominated by all other algorithms, although increasing the number of iterations might boost the performance of LDA. Both STM-TOP and TOP have better performance when one has more anchor words.
	
	\begin{figure}[H]
		\centering
		\includegraphics[width =\textwidth]{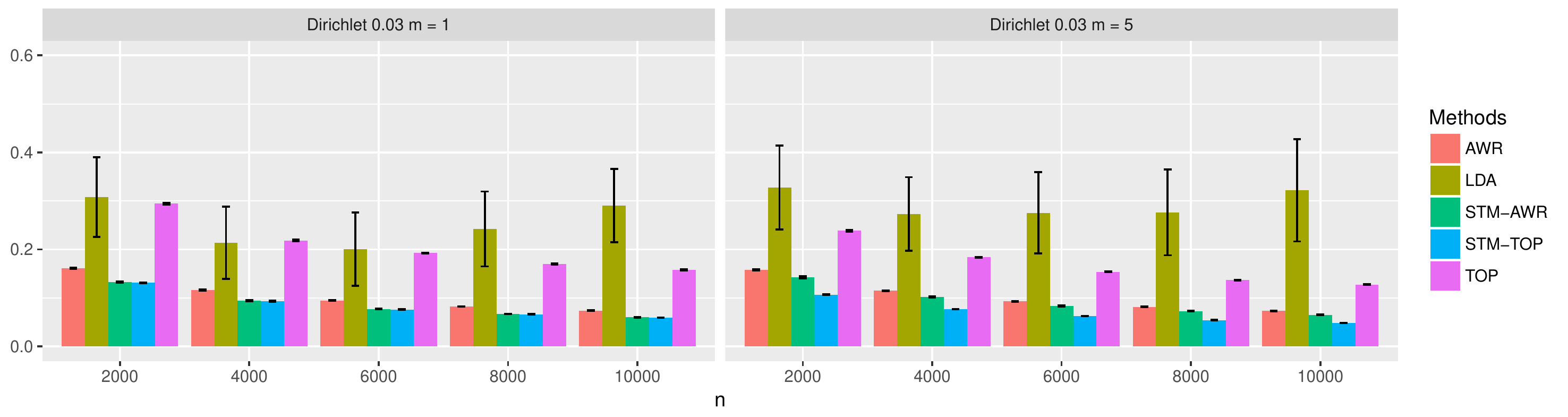}
		\vspace{-5mm}
		\caption{Plots of the estimation errors $\|\wh A -A\|_1/K$}
		\label{fig_nips_diri}
		\vspace{-5pt}
	\end{figure}

	We also investigate   the effect of the correlation among topics on the estimation of $A$. 
	Following \cite{arora2013practical}, we simulate $W$ from a log-normal distribution with block diagonal covariance matrix and different within-block correlation. To construct the block diagonal covariance structure, we divide 100 topics into 10 groups. For each group, the off-diagonal elements of the covariance matrix of topics is set to $\rho$, while the diagonal entries are set to 1. The parameter $\rho \in  \{0.03, 0.3\}$ reflects the magnitude of correlation among topics. We take the case $m = 1$ and the estimation errors of the algorithms are shown in Figure \ref{fig_nips_corr}. \\
	
	\noindent 
	{\bf Conclusion.} STM-TOP has the best performance in all settings. As long as the number of documents $n$ is large, STM-AWR is more robust to the correlation among topics than AWR. LDA and AWR are comparable. 
	
	\begin{figure}[ht]
		\centering
		\includegraphics[width =\textwidth]{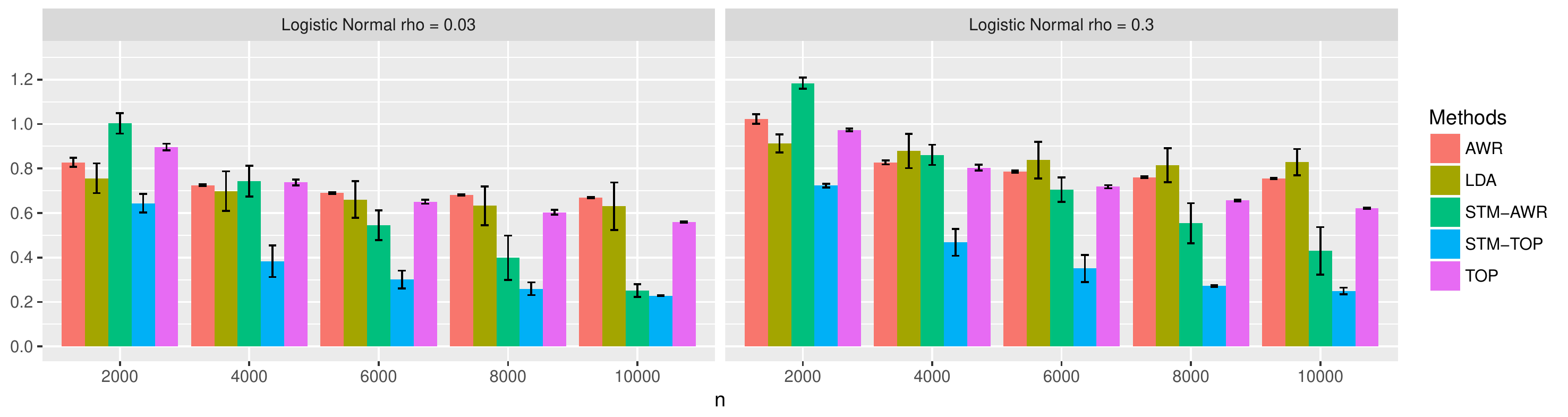}
		\vspace{-5mm}
		\caption{Plots of the estimation errors $\|\wh A -A\|_1/K$ for $\rho = 0.03$ and $\rho = 0.3$.}
		\label{fig_nips_corr}
		\vspace{-5pt}
	\end{figure}
	
	Finally, we report the running times of the various  algorithms in Table \ref{tab_time}. As one can see, LDA is the slowest and does not scale well with $n$. On the other hand, TOP is the fastest and the other three algorithms (AWR, STM-AWR and STM-TOP) have comparable running times.
	
	\begin{table}[ht]
		\centering
		\caption{Running time (seconds) of different algorithms.}
		\label{tab_time}
		\begin{tabular}{lccccc}
			\hline
			& TOP & STM-TOP & AWR & STM-AWR & LDA \\ 
			\hline
			n = 2000 & 35.2 & 614.3 & 393.8 & 500.7 & 1918.7 \\ 
			n = 4000 & 32.8 & 611.2 & 447.0 & 466.2 & 3724.5 \\ 
			n = 6000 & 41.8 & 610.9 & 455.0 & 416.7 & 5616.6 \\ 
			n = 8000 & 44.7 & 605.1 & 458.4 & 463.5 & 7358.8 \\ 
			n = 10000 & 52.0 & 609.0 & 482.8 & 517.9 & 9130.6 \\ 
			\hline
		\end{tabular}
	\end{table}
	
	\subsubsection{New York Times (NYT) dataset}
	After the same preprocessing step, the NYT dataset cotains  $n = 299,419$ documents with dictionary size $p = 3,079$ and mean document length $210$.  We   choose $N = 300$ and vary $n\in\{30000, 40000, \ldots, 70000\}$. The estimated $A^{(0)}$ from LDA has $\textbf{sparsity} \approx 0.679$ calculated from (\ref{s_A}). As in the NIPs corpus earlier, we manually add $|I_k| = m \in \{ 1, 5\}$ anchor words per topic. For each $m$ and $n$, we generate $20$  datasets where columns of $W$ are generated from the symmetric Dirichlet distribution with parameter $0.03$. The average estimation errors $\|\wh A-A\|_1/K$ are shown in Figure \ref{fig_nyt_diri}. We also study the effect of correlation among topics on the estimation errors for the case $m=1$ and with the columns of $W$  generated from the log-normal distribution with block diagonal correlation and $\rho = \{0.1, 0.3\}$. The result is shown in Figure \ref{fig_nyt_corr}.\\ 
	
	\noindent
	{\bf Conclusion.}
		From Figure \ref{fig_nyt_diri}, in the presence of anchor words, we see that  STM-TOP has the best overall performance and STM-AWR outperforms AWR. 
		The errors of STM-TOP and TOP decrease if more anchor words are introduced.
		In Figure \ref{fig_nyt_corr}, STM-TOP outperforms the other algorithms in all cases. TOP has the second best performance while the other three algorithms are comparable.
	
	\begin{figure}[ht]
		\centering
		\includegraphics[width =\textwidth]{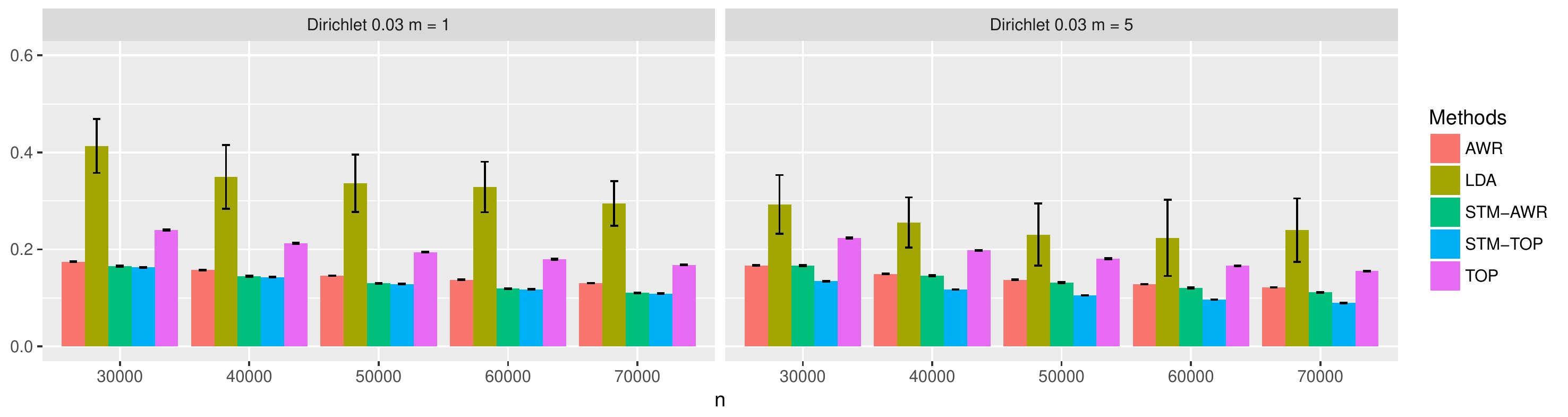}
		\vspace{-5mm}
		\caption{Plots of the estimation errors $\|\wh A -A\|_1/K$}
		\label{fig_nyt_diri}
		 \vspace{-5pt}
	\end{figure}
		\begin{figure}[ht]
		\centering
		\includegraphics[width =\textwidth]{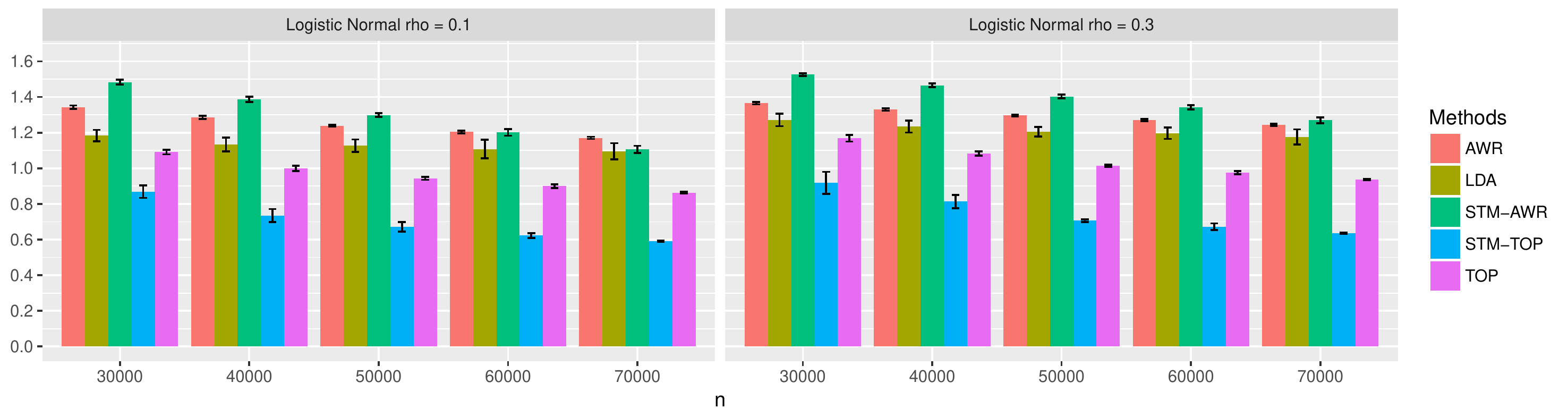}
		\vspace{-5mm}
		\caption{Plots of the estimation errors $\|\wh A -A\|_1/K$ for $\rho = 0.1$ and $\rho = 0.3$.}
		\label{fig_nyt_corr}
	\end{figure}

	\section{Conclusion} We have studied  estimation of the word-topic matrix $A$ when it is possibly entry-wise sparse and the number of topics $K$ is unknown, under the \emph{separability} condition. A new minimax lower bound of $\|\wh A-A\|_1$ is derived and a computationally efficient procedure (STM) for estimating $A$ is proposed.
	The estimator provably achieves the minimax lower bound (modulo a logarithmic factor) and adapts to the unknown sparsity. Extensive simulations corroborate the superior performance of our new estimation procedure in tandem  with the existing algorithm  in \cite{Top} for selecting anchor words.
	
	
	\newpage
	
	\appendix
		
	\section{Proofs}
	The proofs rely on some lemmas  in \cite{Top}. For the reader's convenience, we restate them in Section \ref{sec_proofs_Top} and use similar notations for simplicity. 
	\subsection{Notations and two useful expressions}
	From the topic model specifications, the matrices $\M$, $A$ and $W$ are all scaled as 
	\begin{equation}\label{orig_sum_to_one}
	\sum_{j =1}^p\M_{ji} = 1,\quad\sum_{j =1}^pA_{jk} = 1,\quad 
	\sum_{k=1}^KW_{ki} = 1	
	\end{equation}
	for any $1\le j\le p$, $1\le i\le n$ and $1\le k\le K$. 
	In order to adjust  their scales properly, we denote 
	\begin{equation}\label{def_uag}
	m_{j} = p\max_{1\le i\le n}\M_{ji},~~ \u_j = {p\over n} \sum_{i=1}^n\M_{ji},~~ \alpha_j = p\max_{1\le k\le K} A_{jk}, ~~ \g_k = {K\over n}\sum_{i=1}^nW_{ki},
	\end{equation}
	so that 
	\begin{equation}\label{sum_to_1}
	\sum_{k =1}^K\g_k =K, \qquad \sum_{j =1}^p\u_j = p.
	\end{equation}
	Recall that $\rho_j = \a_j / \ua_L$ and $\wt s_J := \sum_{j\in L^c}\rho_js_j$. We define 
	\begin{equation}\label{mu_hat}
		{\wh \mu_j \over p} = {1\over n}\sum_{t = 1}^n X_{jt},\qquad \text{for all }1\le j\le p.
	\end{equation}
	We write $d:= n\vee p$ throughout the proof. Finally, note that Assumption \ref{ass_w} implies $K<n$.
	
	From model specifications (\ref{orig_sum_to_one}) and (\ref{def_uag}), we derive three useful facts that  are later repeatedly invoked.
	
	\begin{itemize}
		\item[(a)] For any $j\in [p]$, by using (\ref{def_uag}),
		\begin{equation}\label{eq_mu}
		\u_j = {p\over n}\sum_{i =1}^n\M_{ji} = {p\over n}\sum_{i =1}^n\sum_{k =1}^KA_{jk}W_{ki} = {p\over K }\sum_{k =1}^KA_{jk}\g_k ~~\Rightarrow~~{p\over  K} \sum_{k = 1}^KA_{jk}\cdot \ug\le \u_j\le \alpha_j.
		\end{equation}
		In particular, for any $j\in I_k$ with any $k\in [K]$, 
		\begin{equation}\label{eq_mu_I}
		\u_j = {p\over n}\sum_{i =1}^n\sum_{k =1}^KA_{jk}W_{ki} = {p\over K }A_{jk}\g_k\overset{(\ref{def_uag})}{=}{\alpha_j\g_k \over K}.
		\end{equation}
		\item[(b)] For any $j\in [p]$,
		\begin{equation}\label{eq_m}
		m_j \overset{(\ref{def_uag})}{=} {p}\max_{1\le i\le n}\M_{ji} = {p}\max_{1\le i\le n}\sum_{k =1}^KA_{jk}W_{ki}\le  p\max_{1\le k\le K}A_{jk} \overset{(\ref{def_uag})}{=} \alpha_j ~~ \Rightarrow ~~ \u_j \le m_j \le \alpha_j,
		\end{equation}
		by using $0\le W_{ki}\le 1$ and $\sum_k W_{ki} = 1$ for any $k\in[K]$ and $i\in [n]$. 
		\item[(c)] For any $j\in [p]$ and $k\in [K]$, define 
		\begin{equation}\label{def_psi}
		\psi_{jk} = \sum_{a=1}^K A_{ja}C_{ak} \text{ with } C= n^{-1} WW^\top.
		\end{equation}
		We have 
		\begin{equation}\label{eq_psi_jk}\sum_{j=1}^p \psi_{jk} = \sum_{j =1}^p \sum_{a=1}^KA_{ja}C_{ak} = \sum_{a=1}^K C_{ak} =   {1\over n}\sum_{t = 1}^n\sum_{a=1}^KW_{kt}W_{at}
		\overset{(\ref{orig_sum_to_one})}{ =} {1\over n}\sum_{t = 1}^nW_{kt} \overset{(\ref{def_uag})}{= }{\g_k \over K}.
		\end{equation}
	\end{itemize}

	\subsection{Useful results from \cite{Top}}\label{sec_proofs_Top}
	Let $\eps_{ji} := X_{ji} - \M_{ji}$, for $1\le i\le n$ and $1\le j\le p$ and assume $N_1 =\ldots = N_n = N$ for ease of presentation since similar results for different $N$ can be derived by using the same arguments. 
	
	\begin{lemma}\label{lem_t1}
		With probability $1-2d^{-1}$, we have
				\begin{equation}\label{eq_lem_t1}
			{1\over n}\left|\sum_{i =1}^n \eps_{ji} \right| > 2{\sqrt{\u_j\log(d)\over npN}} + {4\log(d) \over nN},	\quad \text{uniformly in $1\le j\le p$.}
				\end{equation}
		If  $\min_{1\le j\le p}\u_j/p\ge \log(d)/(nN)$ holds, with probability $1-2d^{-1}$,
		\[
		{1\over n}\left|\sum_{i =1}^n \eps_{ji} \right| \le 6\sqrt{\u_j\log(d)\over npN},	\quad \text{uniformly in $1\le j\le p$.}
		\]
	\end{lemma}
		
	\begin{lemma}\label{lem_t2} Recall $\Theta=n^{-1} \Pi \Pi^\top$.
		With probability $1-2d^{-1}$,
		\[
		{1\over n}\left|\sum_{i =1}^n \M_{\ell i}\eps_{ji} \right| \le {\sqrt{6m_\ell\Theta_{j\ell}\log(d)\over npN}}+{2m_\ell \log(d) \over  npN},	\quad \text{uniformly in  $1\le j, \ell \le p$.}
		\]
	\end{lemma}
	
	\begin{lemma}\label{lem_t4}
		If $\min_{1\le j\le p}\u_j/p \ge 2\log(d) / (3N)$, then with probability $1-4d^{-1}$,
		\[
		{1\over n}\left|\sum_{i =1}^n\left(\eps_{ji}\eps_{\ell i} - \EE[\eps_{ji}\eps_{\ell i}]\right) \right| \le 12\sqrt{6}\sqrt{\Theta_{j\ell}+{(\u_j + \u_{\ell})\log(d) \over pN}}\sqrt{\log^3(d) \over nN^2} + 4d^{-3},
		\]
		holds, uniformly in  $1\le j,\ell\le p$.
	\end{lemma}	
	
	\begin{lemma}\label{lem_delta}
		Assume model (\ref{model}) and 
		\begin{equation}\label{ass_signal_weak}
		\min_{1\le j\le p} {1\over n}\sum_{i = 1}^n \M_{ji} \ge {c\log(d) \over nN}
		\end{equation}
		for some sufficiently large constant $c>0$. With probability greater than $1- O(d^{-1})$,
		\[
		|\wh \Theta_{j\ell} - \Theta_{j\ell}| \le c_0\eta_{j\ell},\qquad |\wh R_{j\ell} - R_{j\ell}| \le c_1\delta_{j\ell},\quad \text{	for all $1\le  j,\ell \le p$}
		\]
		for some constant $c_0, c_1>0$, where
		\begin{align}\label{def_eta_wt}\nonumber
		\eta_{j\ell} & = \sqrt{\Theta_{j\ell}\log(d)\over nN}\sqrt{{m_j+m_\ell \over p} \vee {\log^2(d) \over N} }+{(m_j+m_\ell)\over p}{ \log(d) \over nN}\\
		&\qquad + \sqrt{\log^4(d) \over nN^3}\sqrt{{\u_j + \u_{\ell}\over p} \vee {\log(d) \over N}}
		\end{align}
		and 
		\begin{equation}\label{delta_wt}
		\delta_{j\ell} := {p^2\eta_{j\ell} \over \u_j \u_{\ell}} +{p^2\Theta_{j\ell}\over \u_j\u_{\ell}}\left(\sqrt{p \over \mu_j} + \sqrt{p\over \mu_\ell}\right) \sqrt{\log(d) \over nN}.
		\end{equation}
	\end{lemma}
	
	\subsection{Proof of Theorem \ref{thm_lb} in Section \ref{sec_lower_bound}}
	We first choose $\{I_1, \ldots, I_K\}$ such that $||I_k| - |I_{k'}|| \le 1$ for any $k, k' \in [K]$. This also implies $|I_k| \le 2|I|/ K$. Further choose the integer set $\{g_1, \ldots, g_K\}$ such that $\sum_{k = 1}^Kg_k = s_J$ and $|g_k - g_{k'}| \le 1$ for any $k, k' \in [K]$, further implying $g_k \le 2s_J / K$. We first choose $A^{(0)}$. 
	Let 
		\begin{equation}\label{def_A0}
		\wt A^{(0)} = 
		\begin{bmatrix}
		\1_{|I_1|} &  & & \\
		& \bm{1}_{|I_2|} &  & \\
		& & \ddots & \\
		& & & \bm{1}_{|I_K|}\\
		\wt \1_{g_1} & \wt \1_{g_2} & \cdots & \wt \1_{g_K}
		\end{bmatrix}
		\end{equation}
		where, for any $k\in [K]$, $\wt \1_{g_k} = \1_{g_k}$ if $g_k = |J|$ and $\wt \1_{g_k} = (\1_{g_k}^\top , 0^\top )^\top $ otherwise. We then set 
		\[
				A^{(0)} = \wt A^{(0)} \begin{bmatrix}
					{1\over |I_1| + g_1} &  & & \\
					& {1\over |I_2| + g_2}  &  & \\
					& & \ddots & \\
					& & & {1\over |I_K| + g_K} \\
				\end{bmatrix}.
		\] 
		We start by constructing a set of ``hypotheses'' of $A$. Assume $|I_k| + g_k$ is even for $1\le k\le K$. Let 
		\[
		\mathcal{M} := \{0,1\}^{(|I|+s_J)/2}.
		\]
		Following the Varshamov-Gilbert bound in Lemma 2.9 in \cite{np_sasha}, there exists $w^{(j)}\in \mathcal{M}$ for $j=0,1,\ldots, T$, such that 
		\begin{equation}\label{eq_w}
		\left\|w^{(i)} - w^{(j)}\right\|_1 \ge {|I|+s_J\over 16},\quad \text{for any } 0\le i\ne j\le T,
		\end{equation}
		with $w^{(0)} = 0$ and	
		\begin{equation}\label{eq_T}
		\log (T) \ge {\log (2) \over 16}(|I|+s_J).
		\end{equation}
		For each $w^{(j)} \in \mathcal{M} $, we divide it into $K$ chunks as
		$
		w^{(j)} = \left(w^{(j)}_1,w^{(j)}_2, \ldots, w^{(j)}_K \right)
		$ 
		with $w^{(j)}_k \in \R^{(|I_k| + g_k)/2}$. 
		For each $w^{(j)}_k$, we write $\wt w^{(j)}_k \in \R^{p}$ as its augumented counterpart such that 
		$[\wt w^{(j)}_k]_{S_k} = [w^{(j)}_k, - w^{(j)}_k]$ and $[\wt w^{(j)}_k]_{\ell} = 0$ for any $\ell \notin S_k$, where $S_k := \textrm{supp}(A^{(0)}_k)$.
		For $1\le j\le T$, we choose $A^{(j)}$ as 
		\begin{equation}\label{eq_Aj}
		A^{(j)} = A^{(0)}+\gamma\begin{bmatrix}
		\wt w^{(j)}_1 & \cdots & \wt w^{(j)}_K
		\end{bmatrix}
		\end{equation}
		with 
		\begin{equation}\label{eq_gamma}
		\gamma = \sqrt{\log (2)\over 4^5(1+c_0)}\sqrt{ K^2 \over nN(|I|+ s_J)}
		\end{equation}
		for some constant $c_0>0$. Under $|I| + s_J \le nN$, it is easy to verify that $A^{(j)} \in \A(|I|, s_J)$ for all $0\le j\le T$.
		
		In order to apply Theorem 2.5 in \cite{np_sasha},  we need to check the following  conditions:
		\begin{itemize}[topsep = 3mm, itemsep = 0ex]
			\item[(a)]  $\KL(\PP_{A^{(j)}},\PP_{A^{(0)}}) \le \log(T)/16$, for each $i= 1,\ldots, T$.
			\item[(b)] $\|A^{(i)}- A^{(j)}\|_1\ge c_1K\sqrt{(|I|+s_J)/(nN)}$, for $0\le i<j\le T$ and some constant $c_1>0$.
		\end{itemize}
		We first show part (a). Fix $1\le j\le T$ and choose $D^{(j)} = A^{(j)}W^{0}$
		where $W^0$ is defined in (\ref{def_W0}). Let $m_k$ be the set such that $|m_k| = n_k$ and
		$
				W_i^0 = e_k,
		$
		for all $i \in m_k$ and $k\in [K]$. 
		Since $|I_k| + g_k \le 2(|I| + s_J)/K$, it follows that
		\begin{equation}\label{eq_d0}
		D^{(0)}_{\ell i} = \sum_{k =1}^KA^{(0)}_{\ell k} W_{ki}^0 = \left\{\begin{array}{ll}
		1 / (|I_k|+g_k) \ge  2^{-1}K/ (|I|+s_J),& \text{ if }\ell \in S_k, i\in m_k, k\in [K]\\
		0, & \text{ otherwise}
		\end{array}\right..
		\end{equation}
		for any $i\in[n]$ and $\ell \in [p]$. Similarly, we have 
		\begin{equation}\label{eq_djd0}
		\left|D^{(j)}_{\ell i}- D^{(0)}_{\ell i}\right| =   \gamma\left|\sum_{k =1}^K [\wt w_k^{(j)}]_{\ell}W_{ki}^{0}\right| \le 
		\left\{\begin{array}{ll}
		\gamma,& \text{ if }\ell \in S_k, i\in m_k, k\in [K]\\
		0,& \text{ otherwise}
		\end{array}\right..
		\end{equation}
		Thus, by $|I| + s_J \le nN$, we have
		$$
		\max_{(\ell,i)\in \mathcal{T}^c}{|D^{(j)}_{\ell i} - D_{\ell i}^{(0)}| \over D_{\ell i}^{(0)}} \le 2\gamma{|I|+s_J \over K} < 1, \qquad \text{for any $1\le j\le T$}
		$$
		where $\mathcal{T} := \{ (\ell , i) \in [p]\times [n]: D_{\ell i}^{(0)} = 0\}$ and $\mathcal{T}^c := [p]\times [n]\setminus \mathcal{T}$.
		Observe that $D^{(j)}_{\ell i} = 0$ for any $(\ell , i)\in \mathcal{T}$ and $1\le j\le T$, and invoke Lemma \ref{lem_KL} to get 
		\begin{align*}
		\KL(\PP_{A^{(j)}}, \PP_{A^{(0)}}) &~\le~ \left(1+c_0\right)N\sum_{(\ell,i)\in \mathcal{T}}{|D^{(j)}_{\ell i} - D^{(0)}_{\ell i}|^2 \over D^{(0)}_{\ell i}}\\
		&~\le ~
		\left(1+c_0\right)N\sum_{k=1}^K\sum_{i \in m_k}\sum_{\ell\in S_k}\g^2 (|I_k| + g_k)\\
		&~=~ \left(1+c_0\right)N\sum_{k=1}^K\sum_{i \in m_k}\g^2 (|I_k| + g_k)^2 \qquad (\text{by }|S_k| = |I_k| + g_k)\\
		&~\le ~ 4\left(1+c_0\right)Nn\g^2 {(|I| + s_J)^2 \over K^2}\\
		&\overset{(\ref{eq_T})}{\le} {1\over 16}\log T.
		\end{align*}
		The second inequality uses (\ref{eq_d0}) and (\ref{eq_djd0}) and the fourth line uses $|I_k| + g_k \le 2(|I|+s_J)/K$. This verifies (a). \\
		
		\noindent 
		To show (b),  (\ref{eq_Aj}) yields
		\begin{align*}
		\|A^{(j)}, A^{(\ell)}\|_1&= \sum_{k =1}^K\left\|A^{(j)}_{\cdot k}-A^{(\ell)}_{\cdot k}\right\|_1\\
		&= 2\gamma \sum_{k =1}^K\left\|w^{(j)}_k-w^{(\ell)}_k\right\|_1\\
		&= 2\gamma \left\|w^{(j)}-w^{(\ell)}\right\|_1\\
		&\overset{(\ref{eq_w})}{\ge} {\gamma \over 8}(|I| + s_J).
		\end{align*}
		After we plug this into the expression of $\gamma$, we obtain (b). Invoking  \cite[Theorem 2.5]{np_sasha} concludes the proof when  $|I_k| + g_k$ is even for all $k\in [K]$. The complementary case is easy to derive with slight modifications. Specifically, denote by $\S_{odd} := \{1\le k\le K: |I_k| + g_k\text{ is odd}\}$. Then we change 
		$
		\mathcal{M}:= \{0, 1\}^{Card}
		$
		with 
		$$
		Card = \sum_{k \in \S_{odd}} {|I_k| + g_k - 1\over 2} + \sum_{k \in \S_{odd}^c} {|I_k| + g_k \over 2}.
		$$
		For each $w^{(j)}$, we write it as $w^{(j)} = (w^{(j)}_1, \ldots, w^{(j)}_K)$ and each $w^{(j)}_k$ has length $(|I_k|+g_k- 1) / 2$ if $k\in \S_{odd}$ and $(|I_k| + g_k) / 2$ otherwise. 	We then construct $
		A^{(j)}_k = A_k^{(0)} + \gamma \wt w^{(j)}_k
		$
		where $\wt w^{(j)}_k\in \R^{p}$ is the same augumented ccounterpart of $w^{(j)}_k$. The result follows from the same arguments and the proof is complete.
		\qed 
	\\
	
	The upper bound of Kullback-Leibler divergence between two multinomial distributions is studied in  \cite[Lemma 6.7]{Tracy}. We   use the following modification of their bound.
	\begin{lemma}\label{lem_KL}
		Let $D$ and $D'$ be two $p\times n$ matrices such that each column of them is a weight vector. Under model (\ref{model}), let $\PP$ and $\PP'$ be the probability measures associated with $D$ and $D'$, respectively. Let $\mathcal{T}$ be the set such that 
		\[
		\mathcal{T} := \left\{(j, i) \in [p] \times [n]: D_{ji} = D'_{ji} =  0\right\}
		\] 
		Let  $\mathcal{T}^c := ([p]\times [n]) \setminus \mathcal{T}$ and 
		$$\eta = \max_{(j, i) \in \mathcal{T}^c}{|D'_{ji} - D_{ji}| \over D_{ji}}$$ and assume $\eta <1$. There exists a universal constant $c_0>0$ such that 
		\[
		\KL(\PP',\PP) \le (1+c_0\eta)N\sum_{(j, i)\in \mathcal{T}^c}{|D'_{ji} - D_{ji}|^2 \over D_{ji}}.
		\]
	\end{lemma}
	\begin{proof}
		With the convention that $0 / 0 = 1$, we have 
		\[
		KL(\PP', \PP) = N\sum_{i = 1}^n\sum_{j =1}^p D'_{ji}\log\left( D'_{ji} \over D_{ji} \right) = N\sum_{(j, i)\in \mathcal{T}^c} D'_{ji}\log\left(1+\eta_{ji} \right).
		\]
		Then the proof follows by the same arugments in \cite{Tracy}.
	\end{proof}

	\subsection{Proofs of Section \ref{sec_upper_bound}}\label{app_sec_A}
	We first give the proof of Lemma \ref{lem_M_hat} and then prove our main Theorem \ref{thm_rate_Ahat}. 
	\subsubsection{Proof of Lemma \ref{lem_M_hat}}
	From (\ref{est_M_h}), we have 
	\[
		\wh M_{ab} = {1\over |L_a||L_b|}\sum_{i\in L_a, j\in L_b}\wh R_{ij}. 
	\]
	Further notice that 
	\[
		{1\over |L_a||L_b|}\sum_{i\in L_a, j\in L_b} R_{ij} = M_{ab}.
	\]
	Using the fact that $\|Q\|_{\rm{\rm{\rm{op}}}} \le \|Q\|_{\r}$  for any symmetric matrix $Q$, yields
	\begin{align*}
			\|\wh M - M\|_{\rm{\rm{\rm{op}}}} &\le \|\wh M - M\|_\r\\
			& = \max_{1\le k\le K}\sum_{a=1}^K \left|{1\over |L_a||L_b|}\sum_{i\in L_a, j\in L_b}(\wh R_{ij} - R_{ij})\right|\\
			&\le \max_{1\le k\le K}\max_{i\in L_k}\sum_{a=1}^K\max_{j\in L_a}|\wh R_{ij} - R_{ij}|.
	\end{align*}
	Invoking Lemma \ref{lem_delta}  for all $i, j\in L$ under condition (\ref{cond_Pi_I}), with probability $1-O(d^{-1})$, we have 
	\begin{align*}
	\|\wh M - M\|_{\rm{\rm{\rm{op}}}} &\le \max_{1\le k\le K}\max_{i\in L_k}\sum_{a=1}^K\max_{j\in L_a}\delta_{ij}.
	\end{align*}
	The result follows by invoking Lemma \ref{lem_delta_II}. \qed

	\subsubsection{Proof of Theorem \ref{thm_rate_Ahat}}
	As our estimation procedure uses a thresholding step in (\ref{est_BJ_1}), we first define 
	\begin{align}\label{def_T}
		T:= \left\{
			j\in L^c:  {1\over n}\sum_{i = 1}^n \M_{ji} <  {\log(d) \over nN}
		\right\},\quad \wh T:= \left\{
		j\in L^c:  {1\over n}\sum_{i = 1}^n X_{ji}<  {7\log(d) \over nN}
		\right\}
	\end{align}
	and write $T^c := [p] \setminus T$
	and $\wh T^c := [p] \setminus \wh T$. 
	\\
	\noindent 
	Recall that our final estimator $\wh A$ is obtained by normalizing $\wh{\bar B} = D_X\wh B$ to unit column sums with 
	$
		D_X = \diag\left({\wh u_1/ p}, \ldots, {\wh u_p/ p}\right)
	$ 
	where $\wh u_j / p$ is defined in (\ref{mu_hat}) for $1\le j\le p$. 
	For any $j\in [p]$ and $k\in [K]$, we have 
	\[
			\wh A_{jk} - A_{jk} = {\wh {\bar B}_{jk} \over \|\wh {\bar B}_k\|_1} - {\bar B_{jk} \over \| \bar B_k\|_1}
	\] 
	where $\bar B = D_\M B$.
	Summing over $1\le j\le p$ yields 
	\begin{align*}
		\|\wh A_k - A_k\|_1 &=  \sum_{j=1}^p \left|
		{\wh {\bar B}_{jk} \over \|\wh {\bar B}_k\|_1} - {\wh {\bar B}_{jk} \over \| \bar B_k\|_1} + {\wh {\bar B}_{jk} - \bar B_{jk} \over \| \bar B_k\|_1}
		\right|\\
		&\le  {|\|\bar B_k\|_1 - \|\wh {\bar B}_k\|_1| \over \|\bar B_k\|_1} + {\|\wh {\bar B}_k - \bar B_k\|_1 \over \|\bar B_k\|_1}\\
		&\le {2 \|\wh {\bar B}_k - \bar B_k\|_1\over\|\bar B_k\|_1} \\
		&= {2K\over \g_k}\|\wh {\bar B}_k - \bar B_k\|_1.
	\end{align*}
In the last equality, we use 
	\[
		\|\bar B_k\|_1 =\sum_{j =1}^p A_{jk} {1\over n}\sum_{t = 1}^n W_{kt} = {\g_k \over K},
	\]
	by observing that $\bar B = D_\M B = AD_W$.
	Further recall that $\wh {\bar B}_{jk} = \wh \u_j \wh B_{jk}/p$ for $j\in [p]$ and  $\wh B_{j\cdot}= 0$ for any $j\in \wh T$. We have 
	\begin{align*}
		\|\wh A_k - A_k\|_1 &={2K \over \g_k}\sum_{j =1}^p \left|
		{\wh \u_j \over p}\wh B_{jk} - {\u_j \over p}B_{jk}
		\right|\\
		&\le {2K \over \g_k}\sum_{j =1}^p \left\{
			\wh B_{jk}{|\wh \u_j - \u_j| \over p} + {\u_j \over p}|\wh B_{jk}-B_{jk}|
		\right\}\\
		&= {2K \over \g_k}\left\{\sum_{j\in \wh T^c} \left(
		\wh B_{jk}{|\wh \u_j - \u_j| \over p} + {\u_j \over p}|\wh B_{jk}-B_{jk}|
		\right) + \sum_{j\in \wh T}{\u_j B_{jk}\over p}\right\}\\
		& = {2K \over \g_k}\sum_{j\in \wh T^c} \left(
		\wh B_{jk}{|\wh \u_j - \u_j| \over p} + {\u_j \over p}|\wh B_{jk}-B_{jk}|
		\right) + 2\sum_{j\in \wh T}A_{jk}.
	\end{align*}
	We use $A_{jk}=(\u_j / p)B_{jk}(K/\g_k)$ in the last line.
	Summing over $1\le k\le K$ gives
	\begin{align*}
		\|\wh A - A\|_1 &\le {2K \over \ug}\sum_{j\in \wh T^c} \left(
		{|\wh \u_j - \u_j| \over p} + {\u_j \over p}\|\wh B_{j\cdot}-B_{j\cdot}\|_1
		\right) + 2\sum_{j\in \wh T} \|A_{j\cdot}\|_1\\
		&= {2K \over \ug} \left\{
		\sum_{j\in \wh T^c} {|\wh \u_j - \u_j| \over p} +\sum_{j\in \wh T^c\setminus L}  {\u_j \over p}\|\wh B_{j\cdot}-B_{j\cdot}\|_1\right\}
		 +2 \sum_{j\in \wh T}\|A_{j\cdot}\|_1.
	\end{align*}
	We use $\|\wh B_{jk}\|_1 = 1$   in the first line and the fact $\wh B_{j\cdot} = B_{j\cdot}$ for all $j\in L$ in the second line. 
	\\
	
	Next, we study the three terms on the right hand side. 
	To bound the first term, we observe that $$\PP\{T\subseteq \wh T\}  = \PP\{\wh T^c \subseteq T^c\}= 1-2d^{-1},$$
	by
 Lemma \ref{lem_mu}. This fact,   the second part of Lemma \ref{lem_t1} and the inequality
	\begin{align}\label{lb_u_T}
		   \min_{j\in T^c} {\u_j \over p} \ge {\log(d)\over nN}.
	\end{align}
	yield
	\[
		\PP\left\{\sum_{j\in \wh T^c} {|\wh \u_j - \u_j| \over p} \le \sum_{j\in T^c} {|\wh \u_j - \u_j| \over p} \le \sum_{j\in T^c}6\sqrt{\u_j \log(d)\over npN} \right\} \ge 1-4d^{-1}.
	\]
	Further, the Cauchy-Schwarz inequality and 
	$\sum_{j\in T^c}\u_j / p \le \sum_{j =1}^p\u_j / p = 1$ yield
	\begin{align}\label{bd_T1}
		\PP\left\{	\sum_{j\in \wh T^c} {|\wh \u_j - \u_j| \over p}   \le 6\sqrt{|T^c|\log(d)\over nN} \le6\sqrt{p\log(d)\over nN}  \right\} \ge 1-4d^{-1}.
	\end{align}
	To bound the third term, Lemma \ref{lem_mu} yields 
	\begin{align}\label{bd_T3}
		 \PP\left\{\sum_{j\in \wh T}\| A_{j\cdot}\|_1\le {20 K|\wh T| \log(d) \over \ug nN} \le {20 Kp\log(d) \over \ug nN} \right\} \ge 1-2d^{-1}.
	\end{align}
	The proof of the upper bound for the second term is more involved. We work on the intersection of the event $\{\wh T^c \subseteq T^c\}$   with 
	\[
			\E_{M}:=\left\{\lambda_{\min}(\wh M +\lambda \bI_K) \ge \lambda_{\min}(M) + \lambda - \|\wh M - M\|_{\rm{\rm{op}}}\ge \lambda_{\min}(M)\right\}
	\]
	to establish an upper bound for
	\[
			\sum_{j\in T^c\setminus L}{\u_j \over p}\|\wh B_{j\cdot} - B_{j\cdot}\|_1.
	\]
	 Lemma \ref{lem_M_hat} and the choice of $\lambda$ guarantee $\PP(\E_{M}) = 1-O(d^{-1})$.
	Pick any $j\in T^c\setminus L$ and recall that $\wh B_{j\cdot}$ is estimated via (\ref{est_BJ_2}). Starting with 
	\[
	\wh B_{j\cdot}^\top (\wh M + \lambda \bI_K)\wh B_{j\cdot} - 2\wh B_{j\cdot}^\top \wh h^{(j)} \le  B_{j\cdot}^\top (\wh M + \lambda \bI_K)^{-1} B_{j\cdot} - 2B_{j\cdot}^\top \wh h^{(j)},
	\]
	standard arguments yield 
	\begin{align*}
		(\Delta^{(j)})^\top (\wh M + \lambda \bI_K)\Delta^{(j)} &\le 2\left|
		(\Delta^{(j)})^\top (\wh h^{(j)} - \wh M B_{j\cdot} - \lambda B_{j\cdot})
		\right|\\
		&\le 2\left\{| (\Delta^{(j)}) ^\top  (\wh h^{(j)} - h^{(j)})| +  |(\Delta^{(j)})^\top (h^{(j)} - \wh M B_{j\cdot})| + \lambda \|\Delta^{(j)}\| \|B_{j\cdot}\|\right\}
	\end{align*}
	by writing  $\Delta^{(j)} := \wh B_{j\cdot} - B_{j\cdot}$.  Hence, on the event $\E_{M}$, we have
	\begin{align}\label{disp_quad}
	\|\Delta^{(j)}\|
	&\le {2 \over \lambda_{\min}(M)}\left\{{| (\Delta^{(j)}) ^\top  (\wh h^{(j)} - h^{(j)})| \over \|\Delta^{(j)}\|}+  {|(\Delta^{(j)})^\top (h^{(j)} - \wh M B_{j\cdot})| \over \|\Delta^{(j)}\|}+ \lambda \|B_{j\cdot}\|\right\}.
	\end{align}
	Let $s_j = \|B_{j\cdot}\|_0$ and $S_j = \textrm{supp}(B_{j\cdot})$. Since
	\[
	0 = \|B_{j\cdot}\|_1 - \|\wh B_{j\cdot}\|_1 = \|B_{jS_j}\|_1 - \|\wh B_{jS_j}\|_1 - \|\wh B_{jS_j^c}\|_1 \le \|\Delta^{(j)}_{S_j}\|_1 - \|\Delta^{(j)}_{S_j^c}\|_1, 
	\]
	we have 
	\begin{equation}\label{eq_Delta}
	\|\Delta^{(j)}\|_1 \le 2\|\Delta^{(j)}_{S_j}\|_1\le 2\sqrt{s_j}\|\Delta^{(j)}_{S_j}\| \le 2\sqrt{s_j}\|\Delta^{(j)}\|.
	\end{equation}
	Combination of (\ref{eq_Delta}) with (\ref{disp_quad})  gives
	\begin{align}\label{disp_quad_2}\nonumber
		&\sum_{j\in T^c\setminus L}{\u_j \over p}\|\wh B_{j\cdot} - B_{j\cdot}\|_1\\\nonumber
		&\le 2\sum_{j\in T^c\setminus L}{\u_j \over p}\sqrt{s_j}\|\Delta^{(j)}\|\\
		& \le {4\over \lambda_{\min}(M)}\sum_{j\in T^c\setminus L}\sqrt{s_j}\cdot {\u_j \over p}\left\{{| (\Delta^{(j)}) ^\top  (\wh h^{(j)} - h^{(j)})| \over \|\Delta^{(j)}\|}+  {|(\Delta^{(j)})^\top (h^{(j)} - \wh MB_{j\cdot})| \over \|\Delta^{(j)}\|}+ \lambda \|B_{j\cdot}\|\right\}
	\end{align}
	 The results of Lemmas \ref{lem_h_hat} and \ref{lem_M_beta}  and the inequality  $\lambda_{\min}(M) \ge \lambda_{\min} K^2 / \og^2$ give 
	\begin{align}\label{bd_T2}\nonumber
	&\sum_{j\in T^c\setminus L}{\u_j \over p}\|\wh B_{j\cdot} - B_{j\cdot}\|_1\\\nonumber
	&\quad \lesssim {\og^2 \over K^2\lambda_{\min}} \Biggl\{
   \max\left\{s_J+|I|-|L|, \wt s_J\right\}\left({K\log(d)\over \ug n N} + \sqrt{p\log^4(d) \over \uu_LnN^3}\right)\\
	& \quad + K\sqrt{\max\left\{s_J+|I|-|L|, \wt s_J\right\}{\log(d) \over \ug nN}}+ \lambda \sum_{j\in T^c\setminus L}\sqrt{s_j}{\u_j \over p}\|B_{j\cdot}\|
	\Biggr\}.
	\end{align}
	Finally, (\ref{bd_T1}), (\ref{bd_T3}) and (\ref{bd_T2}) together imply that
	\begin{align}\label{raw_rate_A}\nonumber
	\|\wh A -A\|_1 & \lesssim
	{K\over \ug}\sqrt{p\log(d) \over nN} + {pK\log(d)\over \ug nN} \\\nonumber
	&\quad + {\og^2 \over \ug K\lambda_{\min}} \Biggl\{
	\max\left\{s_J+|I|-|L|, \wt s_J\right\}\left({K\log(d)\over \ug n N} + \sqrt{p\log^4(d) \over \uu_LnN^3}\right)\\
	& \quad + K\sqrt{\max\left\{s_J+|I|-|L|, \wt s_J\right\}{\log(d) \over \ug nN}}+ \lambda \sum_{j\in T^c\setminus L}\sqrt{s_j}{\u_j \over p}\|B_{j\cdot}\|
	\Biggr\}.
	\end{align} holds
	with probability $1-O(d^{-1})$. After we invoke
	 the result of Lemma \ref{lem_lbd_beta},   the proof of the first result follows. The second result follows by setting $\lambda = 0$ in (\ref{raw_rate_A}) as
	\[
	\PP\left\{\lambda_{\min}(\wh M) \ge \lambda_{\min}(M) - \|\wh M-M\|_{\rm{\rm{op}}} \ge c\lambda_{\min}(M)\right\} \ge 1-O(d^{-1}).
	\] 
	
	\bigskip
	
	\subsection{Lemmas used in the proof of Theorem \ref{thm_rate_Ahat}}

	\begin{lemma}\label{lem_mu}
		Let $T$ and $\wh T$ be defined in (\ref{def_T}). 
		With probability $1-2d^{-1}$, we have $T\subseteq \wh T$ and, for any $1\le j\le p$, if 
		\[
		{1\over n}\sum_{i = 1}^nX_{ji} < {7\log(d)\over nN},
		\]
		we further have
		\[
		\|A_{j\cdot}\|_1 \le  {19 K\log(d) \over \ug nN}.
		\]
	\end{lemma}
	\begin{proof}
		Recall that $X_{ji} = \M_{ji} + \eps_{ji}$ such that $\wh \u_j/p = \u_j/p + n^{-1}\sum_{i = 1}^n\eps_{ji}$. We work on the event 
		\[
		\E_1 := \bigcap_{j=1}^p\left\{ {1\over n}\left|
		\sum_{i = 1}^n\eps_{ji}\right| < 2\sqrt{\u_j\log(d) \over npN} + {4\log(d) \over nN}
		\right\}
		\]
		which holds with probability $1-2d^{-1}$ from Lemma \ref{lem_t1}. Since, for any $j\in T$, 
		\[
		{\wh u_j \over p}\le {\u_j \over p} + {|\wh \u_j - \u_j| \over p}\overset{\E_1}{<} {\log(d) \over nN} + 2\sqrt{\u_j \log(d) \over npN} + {4\log(d) \over nN}< {7\log(d) \over nN},
		\]
		we have $j \in \wh T$, hence $T\subseteq \wh T$. 
		
		To prove the second statement, for any $j$ such that $\wh u_j / p \le 7\log (d) / (nN)$, we have
		\begin{align*}
		{\u_j \over p} \le {\wh \u_j \over p} + {1\over n}\left|\sum_{i= 1}^n\eps_{ji}\right| < {7\log (d) \over nN}+ {1\over n}\left|\sum_{i= 1}^n\eps_{ji}\right|.
		\end{align*}
		For this $j$, since 
		\[
		\PP\left\{
		{1\over n}\left|\sum_{i= 1}^n\eps_{ji}\right|< 2\sqrt{\u_j \log(d)\over npN} + {4\log(d)\over nN}
		\right\}\ge \PP(\E_1) = 1-2d^{-1}, 
		\]
		we have, with probability $1-2d^{-1}$, 
		\[
		{\u_j \over p} < 2\sqrt{\u_j \log(d) \over npN} + {11\log(d) \over nN},
		\]
		which implies $\u_j / p \le 19\log(d) / (nN)$. The result then follows by using (\ref{eq_mu}).
	\end{proof}
	
	\begin{lemma}\label{lem_delta_II}
		Let $\delta_{ij}$ be defined in (\ref{delta_wt}) for any $i,j \in [p]$. 
	Let $	\psi_{jk} $ be defined in (\ref{def_psi}) for any $j\in [p]$, $k\in [K]$.
		Under condition (\ref{cond_Pi_I}),  we have 
		\begin{align*}
		\max_{1\le k\le K}\max_{i\in L_k}\sum_{a=1}^K\max_{j\in L_a}\delta_{ij} &\lesssim {K \over \ug}\sqrt{pK\log(d)\over \uu_L nN}
		\end{align*}
		and, for any $k\in [K]$ and $j\in [p]$,
		\begin{equation*}
		\sqrt{p\log(d)\over \uu_LnN}\max_{i\in L_k}\sum_{a=1}^K{A_{ja}\g_a \over K}\max_{\ell \in L_a}\delta_{i\ell } \lesssim  K\sqrt{\rho_j \psi_{jk} \log(d) \over \ug \g_knN} + \sqrt{\|A_{j\cdot}\|_1}\sqrt{\rho_j \log(d) \over \g_k nN}.
		\end{equation*}
		For any $j\in [p]$, if 
		\[
		{1\over n}\sum_{t=1}^n \M_{jt} \ge {c\log(d)\over nN},
		\]
		for some constant  $c>0$,
		we further have
		\[
		{\u_j \over p} \max_{i\in L_k}\delta_{ij} \lesssim \left(1 + \rho_j\right){K\log(d) \over \g_k nN}+\sqrt{p\log^4(d) \over \uu_L n N^3}+ {K \over \g_k }
		\sqrt{\left(1+ \rho_j\right){\psi_{jk}\log(d) \over nN}} +\sqrt{{\u_j \over p} {K\log(d) \over \g_k n N}}.
		\]
		
	\end{lemma}
	\begin{proof}
		For any $i\in L_k$ and $j\in L_a$ with $a, k \in [K]$, we start with the expressions in (\ref{def_eta_wt}) and (\ref{delta_wt}). Note that (\ref{eq_mu}), (\ref{eq_m}) and (\ref{lb_uL_mL}) imply
		\begin{equation}\label{lb_m_mu}
		{m_i + m_j \over p}\ge { 2c_1\log^2(d)\over p},\quad {\u_i + \u_j \over p}  = {1\over n}\sum_{t = 1}^n (\M_{it} + \M_{jt}) \ge {2c_0 \log(d) \over N}.
		\end{equation}
		Also,   using $m_i \le \alpha_i$ from (\ref{eq_m}) and $\u_i = \alpha_i \g_k / K$ from (\ref{eq_mu_I}), together with  
		\begin{equation}\label{eq_Theta_II}
		\Theta_{ij} = {1\over n}\sum_{t = 1}^nA_{ik}W_{kt}W_{at}A_{ja} =  A_{ik}A_{ja}C_{ka} \overset{(\ref{def_alpha_gamma})}{=} {\alpha_i\alpha_jC_{ka} \over  p^2},
		\end{equation}
		we obtain
		\begin{align}\label{disp_delta_II}\nonumber
		\delta_{ij} &\lesssim {K^2 \over \g_k \g_a}\Biggl\{
		\sqrt{C_{ka}\left({1\over \alpha_i} + {1\over \a_j}\right){p\log(d) \over nN}} + \left({1\over \alpha_i} + {1\over \a_j}\right){p\log(d) \over nN}\\
		&\qquad +\sqrt{\a_i\g_k + \a_j \g_a \over \a_i^2\a_j^2}\sqrt{p^3\log^4(d)\over KnN^3} + C_{ka}\left(\sqrt{1\over \a_i\g_k} + \sqrt{1\over \a_j \g_k}\right)\sqrt{pK\log(d) \over nN}
		\Biggr\}.
		\end{align}
		Using the Cauchy-Schwarz inequality and the fact that 
		\begin{equation}\label{disp_sum_C}
		\sum_{a=1}^KC_{ka} = {1\over n}\sum_{t = 1}^n\sum_{a=1}^KW_{kt}W_{at} = {1\over n}\sum_{t = 1}^nW_{kt} \overset{(\ref{def_alpha_gamma})}{= }{\g_k \over K},
		\end{equation}
		we further have, after a bit of algebra, 
		\[
		\max_{i\in L_k}\sum_{a=1}^K\max_{j\in L_a}\delta_{ij} \lesssim {K^2 \over \ug}\Biggl\{
		\sqrt{p\log(d) \over \ua_L\ug nN} + {pK\log(d) \over \ua_L \ug nN} + \sqrt{p\log(d) \over \ua_L \ug Kn N} + \sqrt{p^3K\log^4(d)\over \ua_L^3\ug^2 nN^3}
		\Biggr\}
		\] 
		where we also use $\a_i \ge \ua_L$, $\g_a \ge \ug$.
		Note that the first term on the right-hand side dominates the other three as 
		\[
		{pK^2\log(d) \over \ua_L\ug nN} \le {1\over c_0},\quad {p^2 K\log^3(d) \over \ua_L^2\ug N^2}\le {p\log^2(d) \over c_0\ua_L N} \le {1\over c_0c_1}
		\]
		by using $K<n$ and the following observation from (\ref{cond_Pi_I}),
		\begin{equation}\label{lb_uL_mL}
				{\ua_L \over p} \overset{(\ref{eq_m})}{\ge} \min_{i\in L}{m_i \over p} \ge {c_1\log^2(d) \over N},\qquad{\ua_L  \over pK} \ge  {\ua_L \ug \over pK} \overset{(\ref{eq_mu_I})}{=} {\uu_L \over p} \ge {c_0\log(d) \over N}.
		\end{equation}
		The first result then follows by using $\uu_L = \ua_L\ug / K$ from (\ref{eq_mu_I}). \\
		
		\noindent 
		To prove the second result, we argue
		\begin{align*}
		&\sum_{a=1}^K {A_{ja}\g_a \over K}\max_{i\in L_k, \ell\in L_a}\delta_{i\ell} \\
		&\lesssim {K\over \g_k}\sum_{a=1}^K A_{ja}\Biggl\{
		\sqrt{C_{ka}p\log(d) \over \ua_L nN} + \sqrt{ p^3\og\log^4(d) \over K\ua_L^3 nN^3} + C_{ka}\sqrt{pK\log(d) \over \ua_L \ug nN} + {p\log(d) \over \ua_L nN}
		\Biggr\}\\
		&\lesssim {K\over \g_k}\sum_{a=1}^K A_{ja}\Biggl\{
		\sqrt{C_{ka}p\log(d) \over \ua_L nN} +  C_{ka}\sqrt{pK\log(d) \over \ua_L \ug nN} + {p\log(d) \over \ua_L nN}
		\Biggr\}\\
		&\le {K\over \g_k}
		\sqrt{\rho_j \psi_{jk}K\log(d) \over  nN} +  K\sqrt{\rho_j \psi_{jk} K\log(d) \over  \ug \g_k nN} + \|A_{j\cdot}\|_1{pK\log(d) \over \ua_L\g_k nN}
		\\
		&\le 2K\sqrt{\rho_j \psi_{jk} K\log(d) \over  \ug \g_k nN} + \|A_{j\cdot}\|_1{pK\log(d) \over \ua_L\g_k nN}
		\end{align*}
		The second line follows   from (\ref{disp_delta_II}),  the third line uses 
		\[	
		{p^2\log^2(d) \over \ua_L^2 n^2N^2} \bigg/  { p^3\og\log^4(d) \over K\ua_L^3 nN^3} = {\ua_L KN \over \og p \log^2(d)} \overset{(\ref{lb_uL_mL})}{\ge} c_1 {K\over \og} \ge c_1,
		\]
		and the fourth line uses the Cauchy-Schwarz inequality together with $\rho_j = \a_j / \ua_L$, $C_{ka}\le \g_k /K$ and (\ref{def_psi}). Since 
		\[
		\sqrt{p\log(d)\over \uu_LnN} \overset{(\ref{lb_uL_mL})}{\ge} \sqrt{1 \over c_0 n} \ge \sqrt{1 \over c_0 K}, 
		\]
		we have 
		\[
		\sqrt{p\log(d)\over \uu_LnN}K\sqrt{\rho_j \psi_{jk} K\log(d) \over  \ug \g_k nN}  \le K\sqrt{\rho_j \psi_{jk}\log(d) \over  c_0\ug \g_k nN}.
		\]
		The result now follows after observing 
 that		\begin{align*}
		\|A_{j\cdot}\|_1{pK\log(d) \over \ua_L\g_k nN}	\sqrt{p\log(d)\over \uu_LnN} &\le 	\|A_{j\cdot}\|_1{pK\log(d) \over \ua_L\ug nN}	\sqrt{pK\log(d)\over \ua_L\g_k nN}\\
		&\le \sqrt{\|A_{j\cdot}\|_1{K\a_j \over p}}{p\log(d) \over \uu_L nN}	\sqrt{pK\log(d)\over \ua_L\g_k nN}\\
		&\overset{(\ref{lb_uL_mL})}{\le} \sqrt{\|A_{j\cdot}\|_1{K\a_j \over p}}{1\over c_0 n}	\sqrt{pK\log(d)\over \ua_L\g_k nN}\\
		&\le \sqrt{\|A_{j\cdot}\|_1}\sqrt{\rho_j \log(d) \over \g_k nN} \quad (\text{by }K<n).
		\end{align*}
		
		\noindent 
		We proceed to prove the third result. Fix any $j\in [p]$ and $i\in L_k$ with $k\in [K]$ and note that (\ref{lb_m_mu}) still holds by replacing the constants $2$ by $1$. Since 
		\begin{equation}\label{eq_Theta_Ij}
		\Theta_{ij} = A_{ik}{1\over n}\sum_{t =1}^n W_{kt}\sum_{a=1}^K A_{ja}W_{at} \overset{(\ref{def_psi})}{=}A_{ik}\psi_{jk},
		\end{equation}
		and   $m_j \le \a_j$ (\ref{eq_m}), $\u_i = \a_i \g_k / K$ 
		from (\ref{eq_mu_I}) and $\rho_j = \a_j / \ua_L$, the
		expressions of (\ref{def_eta_wt}) and (\ref{delta_wt}) yield
		\begin{align}\nonumber
		{\u_j \over p}\delta_{ij} &\lesssim {K \over \g_k }
		\sqrt{\left(1+ \rho_j\right){\psi_{jk}\log(d) \over nN}} + \left(1 + \rho_j\right){K\log(d) \over \g_k nN}+\sqrt{p\log^4(d) \over \uu_L n N^3}\\\nonumber
		&\quad  + {K\over \g_k}\sqrt{{\u_j \over p} {p^2\log^4(d) \over \ua_L^2 n N^3}} + {K\psi_{jk}\over \g_k}\sqrt{pK\log(d)\over \ua_L \g_k n N}+ {K\psi_{jk} \over \g_k}\sqrt{p\log(d)\over \u_j n N}.
		\end{align}
		We now simplify the three terms in the second line. Since \[
		\psi_{jk} = {1\over n}\sum_{t = 1}^n\sum_{a=1}^K A_{ja}W_{at}W_{kt} \le {1\over n}\sum_{t = 1}^n\sum_{a=1}^K A_{ja}W_{at} ={1\over n}\sum_{t = 1}^n\M_{jt}= {\u_j \over p},
		\]
		we have 
		\[
		{K\psi_{jk} \over \g_k}\sqrt{p\log(d)\over \u_j n N} \le {K\over \g_k}\sqrt{\psi_{jk}\log(d)\over nN}.
		\]
		Also note that (\ref{ubd_psi})
		yields 
		\[
		{K\psi_{jk}\over \g_k}\sqrt{pK\log(d)\over \ua_L \g_k n N} \le {K\over \g_k}\sqrt{\a_j \psi_{jk} \log(d) \over \ua_LnN} = {K\over \g_k}\sqrt{\rho_j \psi_{jk} \log(d) \over nN}.
		\]
		Finally, by using
		\[
		{p^2\log^4(d) \over \ua_L^2 n N^3}  \le {p\log^2(d) \over c_1\ua_L nN^2} \le {\g_k \log(d) \over c_0c_1 KnN}
		\]
		from (\ref{lb_uL_mL}) and $\g_k \ge \ug$, we can upper bound $\max_{i\in L_k}(\u_j / p)\delta_{ij} $ by 
		\begin{align}
		\left(1 + \rho_j\right){K\log(d) \over \g_k nN}+\sqrt{p\log^4(d) \over \uu_L n N^3}+ {K \over \g_k }
		\sqrt{\left(1+ \rho_j\right){\psi_{jk}\log(d) \over nN}} +\sqrt{{\u_j \over p} {K\log(d) \over \g_k n N}},
		\end{align}
		which completes the proof. 
	\end{proof}

	The following three lemmas provide upper bounds for the three terms on the right-hand-side of (\ref{disp_quad_2}).  Recall that $\rho_j = \a_j / \ua_L$, $\wt s_J = \sum_{j\in L^c}\rho_j s_j$ and $\psi_{jk} = \sum_{a=1}^KA_{ja}C_{ak}$ for any $j\in [p]$ and $k\in [K]$.
	\begin{lemma}\label{lem_h_hat}
		Under conditions of Theorem \ref{thm_rate_Ahat}, with probability $1-O(d^{-1})$,
		\begin{align*}
		&\sum_{j\in T^c\setminus L}\sqrt{s_j}\cdot {\u_j \over p}{| (\Delta^{(j)}) ^\top  (\wh h^{(j)} - h^{(j)})| \over \|\Delta^{(j)}\|}\\ 
		&\quad \lesssim \max\left\{s_J+|I|-|L|, \wt s_J\right\}\left\{{K\log(d)\over \ug n N} + \sqrt{p\log^4(d) \over \uu_LnN^3}\right\}\\
		& \quad + K\sqrt{\max\left\{s_J+|I|-|L|, \wt s_J\right\}{\log(d) \over \ug nN}}.
		\end{align*}
	\end{lemma}
	\begin{proof}
		Pick any $j\in T^c\setminus L$.  From the definition of $\wh h^{(j)}$ in (\ref{est_M_h}), we have
		\begin{align*}
		{\u_j \over p}| (\Delta^{(j)}) ^\top  (\wh h^{(j)} - h^{(j)})| &\le \sum_{k = 1}^K |\Delta^{(j)}_k|\cdot {\u_j \over p}\left|\wh h^{(j)}_k - h^{(j)}_k\right|\\
		&\le \sum_{k = 1}^K |\Delta^{(j)}_k|\cdot {\u_j \over p}\left|
		{1\over |L_k|}\sum_{i\in L_k}\left(\wh R_{ij} - R_{ij}\right)\right|\\
		&\le c_1\sum_{k = 1}^K |\Delta^{(j)}_k|\cdot {\u_j \over p} \max_{i\in L_k}\delta_{j\ell},
		\end{align*}
		with probability $1-O(d^{-1})$,  invoking Lemma \ref{lem_delta} and  inequality (\ref{lb_u_T}). Application of  the third part of  Lemma \ref{lem_delta_II} further gives
		\begin{align*}
		{\u_j \over p}| (\Delta^{(j)}) ^\top  (\wh h^{(j)} - h^{(j)})| &\le c_1 \|\Delta^{(j)}\| \left[\sum_{k = 1}^K \left(T_2^{(jk)}\right)^2\right]^{1/2} + c_1 \|\Delta^{(j)}\|_1 \max_{1\le k\le K} T_1^{(jk)}\\
		&\overset{(\ref{eq_Delta})}{\le} c_1 \|\Delta^{(j)}\| \left[\sum_{k = 1}^K \left(T_2^{(jk)}\right)^2\right]^{1/2} +2c_1 \sqrt{s_j}\|\Delta^{(j)}\| \max_{1\le k\le K} T_1^{(jk)},
		\end{align*}
		where 
		\begin{align}\label{def_T_jk_1}
		T_1^{(jk)} &=\left(1 + \rho_j\right){K\log(d) \over \g_k nN}+\sqrt{p\log^4(d) \over \uu_L n N^3}\\\label{def_T_jk_2}
		T_2^{(jk)} &=  {K \over \g_k }
		\sqrt{\left(1+ \rho_j\right){\psi_{jk}\log(d) \over nN}} +\sqrt{{\u_j \over p} {K\log(d) \over \g_k n N}}.
		\end{align}
		Hence, by the Cauchy-Schwarz inequality,
		\begin{align*}
		&\sum_{j\in T^c\setminus L}\sqrt{s_j}\cdot {\u_j \over p}{| (\Delta^{(j)}) ^\top  (\wh h^{(j)} - h^{(j)})| \over \|\Delta^{(j)}\|}\\
		&  \lesssim \sqrt{\sum_{j\in T^c\setminus L} (1+\rho_j)s_j} \left\{\sum_{j\in T^c\setminus L}\sum_{k = 1}^K\left({K^2\psi_{jk} \log(d) \over \g_k^2 nN}+ {\u_j K\log(d) \over \g_k npN}\right)\right\}^{1\over 2}\\
		&\quad  + \sum_{j\in T^c\setminus L}s_j \max_{k\in [K]}T_1^{(jk)}.
		\end{align*}
		We conclude our proof by observing that 
		\begin{align*}
		\sum_{j\in T^c\setminus L}s_j &\le s_J+|I|-|L|\\
		\sum_{j\in T^c\setminus L}\sum_{k = 1}^K {K^2\psi_{jk} \over \g_k^2}
		&\le 
		\sum_{k = 1}^K{K^2\over \g_k^2}\sum_{j=1}^p{\psi_{jk}}\overset{(\ref{eq_psi_jk})}{\le} {K^2\over \ug },\\
		\sum_{j\in T^c\setminus L}\sum_{k = 1}^K{\u_j \over p}{K\log(d)\over \g_k nN} &\le {K^2 \log(d) \over \ug nN}
		\end{align*}
		by $\sum_{j=1}^p \u_j = p$. 
	\end{proof}
	
	\begin{lemma}\label{lem_M_beta}
			Under conditions of Theorem \ref{thm_rate_Ahat}, with probability $1-O(d^{-1})$,
			\begin{align*}
			&\sum_{j\in T^c\setminus L}\sqrt{s_j}\cdot {\u_j \over p}{| (\Delta^{(j)}) ^\top  (h^{(j)} - \wh MB_{j\cdot}| \over \|\Delta^{(j)}\|}\\ 
			&\quad \lesssim \wt s_J \sqrt{p\log^4(d)\over \uu_L nN^3} +{K \wt s_J  \log(d) \over  \ug nN}+ K\sqrt{\wt s_J \log(d) \over \ug nN}.
			\end{align*}
	\end{lemma}
	\begin{proof}
      We work on the event 
		\begin{equation}
			\E := \bigcap_{i\in L}\left\{
				 {1\over n}\left|\sum_{t = 1}^n \eps_{it}\right| \le 6\sqrt{\u_i \log(d) \over npN}
			\right\} \bigcap \left\{
			\bigcap_{i, \ell \in L}	\left\{|\wh R_{i\ell} - R_{i\ell}| \le c_1\delta_{i\ell}\right\}
			\right\}.
		\end{equation}
		Lemmas \ref{lem_t1}, \ref{lem_delta} and (\ref{cond_Pi_I}) guarantee that $\PP(\E) \ge 1-O(d^{-1})$. The event $\E$ and (\ref{cond_Pi_I}) further imply 
		\begin{equation}\label{bd_u_hat}
				c {\u_i \over p} \le {\wh u_i \over p} \le c'{\u_i \over p},\qquad \text{for all}\quad i\in L,
		\end{equation}
		for some constants $c, c'>0$ and (\ref{lb_uL_mL}).
		Pick any $j\in T^c\setminus L$ and $k\in [K]$. Observe that $h^{(j)} = MB_{j\cdot}$ and 
		\begin{equation}\label{N1}
		B_{ja} =  {p\over \u_j}A_{ja}{\g_a \over K}.
		\end{equation}
		From (\ref{def_R_hat}) and (\ref{est_M_h}), we have 
		\begin{align*}
			{\u_j \over p}\left|
				(\wh M_{k\cdot} - M_{k\cdot})^\top B_{j\cdot}
			\right| &= {1\over K}\left|
			\sum_{a=1}^K  {A_{ja}\g_a\over |L_k||L_a|}\sum_{i\in L_k, \ell\in L_a} \left(\wh R_{i\ell} -R_{i\ell}\right)
			\right|\\
			&={1\over K}\left|
			\sum_{a=1}^K  {A_{ja}\g_a\over |L_k||L_a|}\sum_{i\in L_k, \ell\in L_a} \left({p^2\wh \Theta_{i\ell} \over \wh u_i \wh u_\ell}- {p^2\Theta_{i\ell} \over \u_i \u_\ell}\right)
			\right|\\
			&\le {1\over K}\left|
			\sum_{a=1}^K  {A_{ja}\g_a\over |L_k||L_a|}\sum_{i\in L_k, \ell\in L_a} {p^2(\wh \Theta_{i\ell} - \Theta_{i\ell})\over  \u_i \u_\ell}\right|\\
			&\quad + 
			 {1\over K}\left|
			\sum_{a=1}^K  {A_{ja}\g_a\over |L_k||L_a|}\sum_{i\in L_k, \ell\in L_a} {(\u_i\u_\ell - \wh \u_i \wh \u_\ell)\over  \u_i \u_\ell}\wh R_{i\ell}\right|\\
			& := {\rm Rem_1}^{(jk)} + {\rm Rem_2}^{(jk)}.
		\end{align*}
	 For ${\rm Rem_2}^{(jk)}$,  we find
	 \begin{align*}
	 	{\rm Rem_2}^{(jk)}&\le \left|
	 	\sum_{a=1}^K  {A_{ja}\g_a\over K}{1 \over  |L_k||L_a|}\sum_{i\in L_k, \ell\in L_a} {[\u_i(\u_\ell- \wh \u_\ell) + (\u_i - \wh \u_i)\wh \u_\ell ]\over  \u_i \u_\ell}\wh R_{i\ell}\right|\\
	 	&\lesssim   
	 	\sum_{a=1}^K  {A_{ja}\g_a\over K}{1 \over  |L_k||L_a|}\sum_{i\in L_k, \ell\in L_a}\wh R_{i\ell}\max_{i\in L_k, \ell\in L_a}\left({|\u_\ell - \wh \u_\ell| \over \u_\ell} + {|\wh \u_i - \u_i| \over \u_i}\right) \\
	 	&\lesssim \sum_{a=1}^K  {A_{ja}\g_a\over K}\left(\sqrt{pK\log(d)\over \ua_L\g_k nN} + \sqrt{pK\log(d)\over \ua_L\g_a nN}\right){1 \over  |L_k||L_a|}\sum_{i\in L_k, \ell\in L_a}(R_{i\ell} + c_1 \delta_{i\ell})\\
	 	&\le 
	 	2\sum_{a=1}^K  {A_{ja}C_{ka} K\over \g_k}\sqrt{pK\log(d)\over \ua_L\ug nN}+ \sqrt{p\log(d)\over \uu_LnN}\sum_{a=1}^K  {A_{ja}\g_a\over K}\max_{i\in L_k, \ell\in L_a}\delta_{i\ell}.
		 \end{align*}
We use (\ref{bd_u_hat}) in the second line, the definition of the event $\E$ together with (\ref{eq_mu_I}) in the third line and
	 \[
	 		R_{i\ell} = {p^2\Theta_{i\ell} \over \u_i \u_\ell} = {K^2C_{ka} \over \g_k \g_a}
	 \]
	(follows  from (\ref{eq_mu_I}) and (\ref{eq_Theta_II})) in the fourth line. We bound the first term on the right as
	 \begin{align*}
	 \sum_{a=1}^K  {A_{ja}C_{ka} K\over \g_k}\sqrt{pK\log(d)\over \ua_L\ug nN} =   { \psi_{jk} K\over \g_k}\sqrt{pK\log(d)\over \ua_L\ug nN} \le K\sqrt{\rho_j\psi_{jk}\log(d)\over \ug\g_k  nN} 
	 \end{align*}
	 by using
	 \begin{equation}\label{ubd_psi}
	 \psi_{jk} = {1\over n}\sum_{t = 1}^n\sum_{a=1}^K A_{ja}W_{at}W_{kt} \le \|A_{j\cdot}\|_\i {1\over n}\sum_{t = 1}^n\sum_{a=1}^K W_{at} W_{kt}={\a_j \g_k\over pK}.
	 \end{equation}
	 Invoking the second result of Lemma \ref{lem_delta_II} gives
	 \begin{equation}\label{bd_Rem_2}
	 	{\rm Rem_2}^{(jk)} \lesssim
	 		K\sqrt{\rho_j\psi_{jk}\log(d)\over \ug\g_k  nN} + \sqrt{\rho_j\|A_{j\cdot}\|_1\log(d) \over \g_k n N}. 
	 \end{equation}
	 
	 \noindent 
	 We proceed to bound ${\rm Rem_1}^{(jk)}$. Recalling (\ref{N1}) and  $\u_{\ell} / p = A_{\ell a}\g_a / K$ from (\ref{eq_mu_I}), we find
	 \begin{align*}
	 	{\rm Rem_1}^{(jk)} & = \left|
	 		\sum_{a=1}^K  {A_{ja}\over |L_k||L_a|}\sum_{i\in L_k, \ell\in L_a} {p(\wh \Theta_{i\ell} - \Theta_{i\ell})\over  \u_i A_{\ell a}}\right|\\
	 		&\le \max_{i\in L_k}{p\over \u_i}\left|
	 		\sum_{a=1}^K  {1 \over |L_a|}\sum_{\ell\in L_a} {A_{ja}\over A_{\ell a}}(\wh \Theta_{i\ell} - \Theta_{i\ell})\right|.
	 \end{align*}
	 Since, for any $i\in L_k$, $j\in L_a$, 
	 \begin{align*}
	 	\wh \Theta_{i\ell} - \Theta_{i\ell} &= {N\over N-1}\left({1\over n}A_{ik}W_{k}^\top \eps_\ell + {1\over n}A_{\ell a}W_a^\top \eps_i\right) + {N\over N-1}\left({1\over n}\eps_i^\top \eps_\ell - {1\over n}\EE\left[\eps_i^\top \eps_\ell\right] \right)\\
	 	&\quad  - {1\over N-1}\diag\left({1\over n}\sum_{t=1}^{n}\eps_{it} \right)1_{\{i = \ell\}},
 	 \end{align*}
 	 cf.  \cite[page 11 in the Supplement]{Top}, we obtain
 	  \begin{align}\label{def_Rem_1}\nonumber
 	{\rm Rem_1}^{(jk)} &\lesssim \max_{i\in L_k}{p\over \u_i}\Biggl\{
 	 		 A_{ik}\left|
 	 		 	\sum_{a=1}^K  {1 \over |L_a|}\sum_{\ell\in L_a} {A_{ja}\over A_{\ell a}}{1\over n}\sum_{t=1}^nW_{kt}\eps_{\ell t}
 	 		 \right|+\left|
 	 		 \sum_{a=1}^K  A_{ja}{1\over n}\sum_{t=1}^{n}W_{at}\eps_{it}
 	 		 \right|\\\nonumber
 	 		 &\quad\qquad +
 	 		 \left|
 	 		 	\sum_{a=1}^K  {1 \over |L_a|}\sum_{\ell\in L_a} {A_{ja}\over A_{\ell a}}\left({1\over n}\eps_i^\top \eps_\ell - {1\over n}\EE\left[\eps_i^\top \eps_\ell\right] \right)
 	 		 \right|+ {A_{jk}\over NA_{ik}}\left|{1\over n}\sum_{t = 1}^n \eps_{it}\right|
 	 \Biggr\}\\
 	 & := {\rm Rem_{11}}^{(jk)} + {\rm Rem_{12}}^{(jk)} + {\rm Rem_{13}}^{(jk)} + {\rm Rem_{14}}^{(jk)}.
 	 \end{align}
 	 In the sequel, we provides separate   bounds the each of the four terms. We start with the last term and obtain on the event $\E$  
 	 \begin{equation}\label{bd_Rem_14} 	 
 	 		{\rm Rem_{14}}^{(jk)} \le \max_{i\in L_k}{p\over \u_i}{A_{jk}\over NA_{ik}}\left|{1\over n}\sum_{t = 1}^n \eps_{it}\right| \le {6\rho_j} \sqrt{p\log(d)\over \uu_L nN^3}.
 	 \end{equation}
 	 by recalling that $\rho_j = \a_j / \ua_L$.
 	 Observing that $\sum_a A_{ja}W_{at} = \M_{jt}$, with probability $1-O(d^{-1})$, the second term can be upper bounded by using Lemma \ref{lem_t2} as 
 	 \begin{align}\label{bd_Rem_12} \nonumber	 
 	 			{\rm Rem_{12}}^{(jk)}  & = \max_{i\in L_k}{p\over \u_i}\left|{1\over n}\sum_{t = 1}^n\M_{jt}\eps_{it} \right|\\\nonumber
 	 			&\le \max_{i\in L_k}{p\over \u_i}\left({\sqrt{6m_j\Theta_{ji}\log(d)\over npN}}+{2m_j \log(d) \over  npN}\right)\\\nonumber
 	 			&\le \max_{i\in L_k}\left({K\over \g_k}{\sqrt{6\a_j \psi_{jk}\log(d)\over \a_i nN}}+{2\a_j \log(d) \over  \u_i nN}\right)\\
 	 			&\le {K\over \g_k}{\sqrt{6\rho_j \psi_{jk}\log(d)\over  nN}}+{2\rho_j K \log(d) \over  \g_k nN}
 	 \end{align}
 	 where we also use(\ref{eq_m}) and (\ref{eq_Theta_Ij}) to derive the third line and use (\ref{eq_mu_I}) to arrive at the last line. 
 	 The upper bounds of  ${\rm Rem_{11}}^{(jk)}$ and ${\rm Rem_{13}}^{(jk)}$ are proved in Lemmas \ref{lem_Rem_11} and \ref{lem_Rem_13}. Combination of (\ref{bd_Rem_2}), (\ref{bd_Rem_14}), (\ref{bd_Rem_12}), (\ref{bd_ Rem_11}) and (\ref{bd_ Rem_13}) yields
 	\begin{align*}
 				{\u_j \over p}\left|
 				(\wh M_{k\cdot} - M_{k\cdot})^\top B_{j\cdot}
 				\right| &\lesssim 
 				  {\rho_j} \sqrt{p\log^4(d)\over \uu_L nN^3} +{2\rho_j K \log(d) \over  \g_k nN}\\
 				&\quad +K\sqrt{\rho_j\psi_{jk}\log(d)\over \ug\g_k  nN} +\sqrt{\rho_j\|A_{j\cdot}\|_1\log(d) \over \g_k n N} + \sqrt{{\u_j \over p}{\rho_jK \log(d)\over \g_k n N}} 
 		\end{align*} 
 	with probability $1-O(d^{-1})$. Next we use similar arguments as in the proof of Lemma \ref{lem_h_hat}. Analogous to (\ref{def_T_jk_1}) -- (\ref{def_T_jk_2}), we define 
 	\begin{align*}
 		\rho_j R^{(k)}_{1} &=   \sqrt{p\log^4(d)\over \uu_L nN^3} +{2 K \log(d) \over  \g_k nN}\\
 		\sqrt{\rho_j}R^{(jk)}_{2} &=  K\sqrt{\psi_{jk}\log(d)\over \ug\g_k  nN} +\sqrt{\|A_{j\cdot}\|_1\log(d) \over \g_k n N} + \sqrt{{\u_j \over p}{K \log(d)\over \g_k n N}}.
 	\end{align*}
 	We can obtain 
 	\[
 			{\u_j \over p}\left|(\Delta^{(j)})^\top 
 			(\wh M - M)B_{j\cdot}
 			\right| \lesssim \|\Delta^{(j)}\|\sqrt{\rho_j} \left[\sum_{k = 1}^K \left(R_2^{(jk)}\right)^2\right]^{1/2} + \rho_j\sqrt{s_j}\|\Delta^{(j)}\| \max_{1\le k\le K} R_1^{(k)}
 	\]
 	which, by the Cauchy-Schwarz inequality, further gives 
 	\begin{align*}
 		&\sum_{j\in T^c\setminus L}\sqrt{s_j}\cdot {\u_j \over p}{| (\Delta^{(j)}) ^\top  (h^{(j)} - \wh MB_{j\cdot}| \over \|\Delta^{(j)}\|}\lesssim \sqrt{\wt s_J} \left[\sum_{j\in T^c\setminus L}\sum_{k = 1}^K \left(R_2^{(jk)}\right)^2\right]^{1\over 2}+ \wt s_J\max_{k\in [K]}R_1^{(k)}.
 	\end{align*}
 	Finally, we calculate $\sum_{j\in T^c\setminus L}\sum_{k = 1}^K (R_2^{(jk)})^2$ as 
 	 \begin{align*}
 	 		&\sum_{j\in T^c\setminus L}\sum_{k = 1}^K (R_2^{(jk)})^2\\
 	 		&\le \sum_{j\in T^c\setminus L}\sum_{k = 1}^K \left\{{K^2\psi_{jk}\log(d)\over \ug\g_k  nN} +{\|A_{j\cdot}\|_1\log(d) \over \g_k n N} + {{\u_j \over p}{\rho_jK \log(d)\over \g_k n N}}\right\}\\
 	 		&\le {3K^2\log(d)\over \ug  nN} .
 	 \end{align*}
 	 We use (\ref{sum_to_1}), (\ref{eq_psi_jk}) and $\sum_{j=1}^p\|A_{j\cdot}\|_1 = K$ to arrive at the last line.
	\end{proof}

	\begin{lemma}\label{lem_lbd_beta}
		     Let $\lambda$ be chosen as in (\ref{rate_lambda_thm}). With probability $1-O(d^{-1})$, 
		     \[
		     	\lambda \sum_{j\in T^c\setminus L}{\u_j \over p}\sqrt{s_j}\|B_{j\cdot}\| \le 	cK \sqrt{{\og \over \ug}\cdot {K\wt s_J\log(d) \over \ug nN}}. 
		     \]
	\end{lemma}
	\begin{proof}
		    Recall that $B_{jk}\u_j / p = A_{jk}\g_k / K$. We have 
		    \begin{align*}
		    		\sum_{j\in T^c\setminus L}{\u_j \over p}\sqrt{s_j}\|B_{j\cdot}\|&= 	{1\over K}\sum_{j\in T^c\setminus L}\sqrt{s_j}\left[
		    			\sum_{k = 1}^K A_{jk}^2\g_k^2 
		    		\right]^{1/2}\\
		    		&\le  {1\over K}\sum_{j\in T^c\setminus L}\sqrt{s_j}\left[
		    		\sum_{k = 1}^K A_{jk}\g_k
		    		\right]^{1/2}\sqrt{{\a_j \og \over p}}.
		    \end{align*}
		    From	$\uu_L = \ua_L\ug / K$ and the choice of $\lambda$, it follows that, with   probability $1-O(d^{-1})$,  
		    \begin{align*}	
		    			\lambda \sum_{j\in T^c\setminus L}{\u_j \over p}\sqrt{s_j}\|B_{j\cdot}\| 
		    		&\le  {c\over \ug}\sqrt{pK^2\log(d) \over \ua_L \ug nN}  \sum_{j\in T^c\setminus L}\sqrt{s_j}\left[
		    		\sum_{k = 1}^K A_{jk}\g_k
		    		\right]^{1/2}\sqrt{{\a_j \og \over p}}\\
		    		& = c\sqrt{\og K^2\log(d) \over  \ug^3nN}  \sum_{j\in T^c\setminus L}\sqrt{\rho_js_j}\left[
		    		\sum_{k = 1}^K A_{jk}\g_k
		    		\right]^{1/2}\\
		    		&\le cK\sqrt{\og K\log(d) \over  \ug^3nN} \sqrt{\wt s_J}\left[ \sum_{j\in T^c\setminus L}\sum_{k = 1}^K{A_{jk}\g_k \over K}
		    		\right]^{1/2}\\
		    		&=  cK\sqrt{\og \wt s_J K\log(d) \over  \ug^3nN} \left[ \sum_{j\in T^c\setminus L}{\u_j \over p}
		    		\right]^{1/2}\\
		    		&\le cK\sqrt{\og \wt s_J K\log(d) \over  \ug^3nN} .
		    \end{align*}
			Here we use the Cauchy-Schwarz inequality in the third line and the identity $\sum_{j =1}^p \u_j = p$ in the last line. 
			This completes the proof.
	\end{proof}
	
	\subsection{Lemmas used in the proof of Lemma \ref{lem_M_beta}}
	
	Let ${\rm Rem_{11}}^{(jk)}$ and ${\rm Rem_{13}}^{(jk)}$, $j\in [p]$, $k\in [K]$,  be defined as (\ref{def_Rem_1}). 
	
	\begin{lemma}\label{lem_Rem_11}
		Under conditions of Theorem \ref{thm_rate_Ahat}, with probability $1-2d^{-1}$, 
		\begin{equation}\label{bd_ Rem_11}
		{\rm Rem_{11}}^{(jk)} \le {K\over \g_k}\sqrt{6\rho_j\psi_{jk} \log(d)\over nN} + {4\rho_j K\log(d) \over \g_k nN}
		\end{equation}
		uniformly for any $j\in [p]$ and $k\in [K]$.
	\end{lemma}
	
	\begin{proof}
		We upper bound ${\rm Rem_{11}}^{(jk)}$ by studying 
		\[
		\max_{i\in L_k}{p\over \u_i}
		A_{ik}\left|
		\sum_{a=1}^K  {1 \over |L_a|}\sum_{\ell\in L_a} {A_{ja}\over A_{\ell a}}{1\over n}\sum_{t=1}^nW_{kt}\eps_{\ell t}
		\right| \overset{(\ref{eq_mu_I})}{=} {K\over \g_k}\left|
		{1\over n}\sum_{t=1}^nW_{kt}\sum_{a=1}^K\sum_{\ell\in L_a} {1 \over |L_a|} {A_{ja}\over A_{\ell a}}\eps_{\ell t}
		\right|.
		\]
		Recall that 
		\begin{equation}\label{decomp_eps}
		\eps_{\ell t} = {1\over N}\sum_{r = 1}^N Z_{rt}^{(\ell)}
		\end{equation}
		where $Z_{rt}^{(\ell)}$ denotes the $\ell$th element of $Z_{rt}$ and $Z_{rt}$ has a centered $ \text{Multinomial}_p(1; \M_t)$ (subtracted its mean  $M_t$). Next we will   use   Bernstein's inequality to   bound 
		\begin{equation}\label{def_zeta}
		\left|
		{1\over n}\sum_{t=1}^n\sum_{r=1}^NW_{kt}\left(
		\sum_{a=1}^K\sum_{\ell\in L_a} {1 \over |L_a|} {A_{ja}\over A_{\ell a}}Z_{rt}^{(\ell)}
		\right)
		\right| := \left|
		{1\over n}\sum_{t=1}^n\sum_{r=1}^NW_{kt}\zeta_{rt}
		\right| 
		\end{equation} from above.
		Note that $\EE[W_{kt}\zeta_{rt}] = 0$ and
		\begin{equation}\label{bd_zeta}
		\left|
		W_{kt}\zeta_{rt}
		\right| \le \rho_j 	\sum_{a=1}^K\max_{\ell \in L_a}\left|Z_{rt}^{(\ell)} \right| \le 2\rho_j.
		\end{equation}
		To calculate the variance of $\sum_{t = 1}^n\sum_{r=1}^{N} W_{kt}\zeta_{rt}$, observe that 
		\[
		\zeta_{rt} = \eta^\top  Z_{rt}^{L}
		\]
		with $Z_{rt}^{L}$ denoting the sub-vector of $Z_{rt}$ corresponding to $L$ and 
		\begin{equation}\label{def_eta}
		\eta  = 
		D_L\begin{bmatrix}
		\1_{|L_1|} & & \\
		& \ddots & \\
		& & \1_{|L_K|}
		\end{bmatrix}\begin{bmatrix}
		{A_{j1} / |L_1|} \\ \vdots \\ {A_{jK} / |L_K|}
		\end{bmatrix}\in \R^{|L|}
		\end{equation}
		where $(D_L)_{\ell\ell} = 1/A_{\ell a}$ for any $\ell \in L_a$ and $a\in [K]$. We thus have
		\begin{align}\nonumber
		{\rm Var}\left(
		\sum_{t = 1}^n\sum_{r=1}^{N} W_{kt}\zeta_{rt}
		\right) &\le N\sum_{t = 1}^n W_{kt}^2 \eta^\top \diag(\M_{Lt})\eta\\\nonumber
		& =  nN {1\over n}\sum_{t = 1}^n W_{kt}^2 \sum_{a=1}^K\sum_{\ell\in L_a} \M_{\ell t}\left({A_{ja} \over A_{\ell a}} {1\over |L_a|}\right)^2\\\nonumber
		&\le nN {1\over n}\sum_{t = 1}^n W_{kt} \sum_{a=1}^K{1\over |L_a|}\sum_{\ell\in L_a} W_{a t}{A_{ja}^2 \over A_{\ell a}} \\
		&\le \rho_jnN\psi_{jk},
		\end{align}
		using $W_{kt}\le 1$ and $|L_a| \ge 1$ in the third line and  (\ref{def_psi}) in the last line. Invoke Lemma \ref{lem_bernstein} with $B = 2\rho_j$ and $v = \rho_jnN\psi_{jk}$ to obtain, for any $t>0$,
		\begin{align*}
		\PP\left\{
		{1\over nN}\left|
		\sum_{t = 1}^n\sum_{r = 1}^N W_{kt}\zeta_{rt}
		\right| > t
		\right\} \le 2\exp\left(
		-{n^2N^2 t^2 / 2 \over \rho_jnN\psi_{jk}+ 2nN\rho_j t / 3}
		\right),
		\end{align*}
		which further implies
		\begin{align*}
		\PP\left\{
		{1\over nN}\left|
		\sum_{t = 1}^n\sum_{r = 1}^N W_{kt}\zeta_{rt}
		\right| > \sqrt{\rho_j\psi_{jk} t\over nN} + {2\rho_j t \over 3nN}
		\right\} \le 2e^{-t/2}
		,\quad \text{for any }t>0.
		\end{align*}
		Choosing $t = 6\log(d)$ yields
		\begin{equation*}
		{\rm Rem_{11}}^{(jk)} \le {K\over \g_k}\sqrt{6\rho_j\psi_{jk} \log(d)\over nN} + {4\rho_j K\log(d) \over \g_k nN}
		\end{equation*}
		with probability $1-2d^{-3}$. Taking the union bound for probabilities completes the proof. 
	\end{proof}
	
	\begin{lemma}\label{lem_Rem_13}
		Under conditions of Theorem \ref{thm_rate_Ahat}, with probability $1-6d^{-1}$, we have
		\begin{equation}\label{bd_ Rem_13}
		{\rm Rem_{13}}^{(jk)} \lesssim  {K\over \g_k}\sqrt{\rho_j\psi_{jk} \log(d)\over nN} + \sqrt{{\u_j \over p}{\rho_jK \log(d)\over \g_k n N}} + \rho_j\sqrt{p\log^4(d)\over \uu_L nN^3}
		\end{equation}
		uniformly for $j\in [p]$ and $k\in [K]$. 
	\end{lemma}
	\begin{proof}
		Recall that 
		\begin{align*}
			{\rm Rem_{13}}^{(jk)}
			&= \max_{i\in L_k}{p\over \u_i}\left|
			{1\over n}\sum_{t = 1}^n\left(\eps_{it}\xi_t- \EE \left[\eps_{it}\xi_t\right]\right)
			\right|.
		\end{align*}
		Using (\ref{decomp_eps}) and (\ref{def_zeta}), we have  
		\[
			\xi_t := \sum_{a=1}^{K}\sum_{\ell\in L_a}{A_{ja} \over A_{\ell a} |L_a|}\eps_{\ell t} ={1\over N}\sum_{r=1}^N\zeta_{rt}.
		\]
	    We will use  similar truncation arguments in tandem with Hoeffding's inequality as   in \cite[proof of Lemma 15]{Top}. This implies that, for any $i\in L$, 
	    \[
	     \PP\left\{  \left| \eps_{it} \right| \le {\sqrt{6\M_{it}\log(d)\over N}} + {2\log(d)\over N} := T_t\right\} \ge 1-2d^{-3}.
	    \]
	    To truncate $\zeta_{rt}$, recall that $\EE[\zeta_{rt}] = 0$, $|\zeta_{rt}| \le 2\rho_j$ from (\ref{bd_zeta}) and 
	    \begin{align*}
	      	{\rm Var}\left(\sum_{r = 1}^N\zeta_{rt}\right) &= N 	 \eta^\top \diag(\M_{Lt})\eta\\\nonumber
	      	& =  N  \sum_{a=1}^K\sum_{\ell\in L_a} \M_{\ell t}\left({A_{ja} \over A_{\ell a}} {1\over |L_a|}\right)^2\\\nonumber
	      	&\le N  \sum_{a=1}^K{1\over |L_a|}\sum_{\ell\in L_a} W_{a t}{A_{ja}^2 \over A_{\ell a}} \\
	      	&\le N\rho_j  \M_{jt}
	    \end{align*}
	    where $\eta$ is defined in (\ref{def_eta}). Invoking Lemma \ref{lem_bernstein} with $B = 2\rho_j$ and $v = N\rho_j  \M_{jt}$ yields
	    \begin{align*}
	    \PP\left\{
	    {1\over N}\left|
	   \sum_{r = 1}^N \zeta_{rt}
	    \right| \le \sqrt{6\rho_j \M_{jt} \log(d)\over N} + {4\rho_j \log(d) \over N} := T'_{t}
	    \right\} \ge 1-2d^{-3}.
	    \end{align*}
	     We define 
			$
			Y_{it} = \eps_{it}\1_{\S_{t}}
			$ with
			$
			\S_{t} := \left\{|\eps_{it}| \le T_t \right\}
			$
			and $Y'_{t} = \xi_t \1_{\S'_t}$ with $\S'_t := \left\{|\xi_t| \le T'_t\right\}$, 
			for each $i\in [p]$ and  $t\in [n]$, and set 
			$
			\S := \cap_{i=1}^p\cap_{t=1}^n\S_{t} \cap \S'_t.
			$
			It follows that 
			$\PP(\S ) \ge 1- 4d^{-1}$. 
			On the event $\S$, we have
			\[
			{1\over n}\left|\sum_{t =1}^n\left(\eps_{it}\xi_t- \EE[\eps_{it}\xi_t]\right) \right|\le 
			\underbrace{{1\over n}\left|\sum_{t =1}^n\left(Y_{it}Y'_t- \EE[Y_{it}Y'_t]\right) \right|}_{R_1}+ \underbrace{{1\over n}\left|\sum_{t =1}^n\left(\EE[\eps_{it}\xi_t]- \EE[Y_{it}Y'_t]\right) \right|}_{R_2} 
			\]
			Since
			\begin{align*}
			\EE[\eps_{it}\xi_t] &= \EE[Y_{it}Y'_t]+\EE\left[Y_{it}\xi_t\1_{(\S'_{t})^c}\right]+\EE\left[\eps_{it}\1_{\S_t^c}\xi_t\right],
			\end{align*}
			we have
			\begin{align}\nonumber
			R_2 ={1\over n}\left|\sum_{t =1}^n\left(\EE[\eps_{it}\xi_t]- \EE[Y_{it}Y'_t]\right) \right| & \le {1\over n}\left|\sum_{t =1}^n\left(\EE\left[Y_{it}\xi_t\1_{(\S'_t)^c}\right]+\EE\left[\eps_{it}\1_{\S_{t}^c}\xi_t\right]\right) \right|\\\label{eq_T2}
			&\le {1\over n}\sum_{t =1}^n2\rho_j\left(\PP(\S_{t}^c)+\PP((\S'_t)^c) \right)\\ &\le 8\rho_j d^{-3}\nonumber
			\end{align}
			by using $|Y_{it}| \le |\eps_{it}| \le 1$ and $|\xi_t| \le |\zeta_{rt}| \le 2\rho_j$ in the second inequality. 
			
			It remains to bound $R_1$. Since $|Y_{it}| \le T_t$, we know $-2T_tT'_t\le Y_{it}Y'_t - \EE[Y_{it}Y'_t]\le 2T_tT'_t$ for all $1\le t\le n$. Applying   Hoeffding's inequality (Lemma \ref{hoeff}) with $a_t = -2T_tT'_{t}$ and $b_t = 2T_{t}T'_{t}$ gives 
			\[
			\PP\left\{ \left|\sum_{t=1}^n\left(Y_{it}Y'_t- \EE[Y_{it}Y'_{t}]\right) \right| \ge t \right\} \le 2\exp\left(-{t^2 \over 8\sum_{t =1}^nT_{t}^2(T'_t)^2}\right).
			\] 
			Taking $t = \sqrt{24\sum_{t =1}^nT_{t}^2(T'_t)^2\log(d)}$ yields
			\begin{equation}\label{eq_T1}
			R_1 = {1\over n}\left|\sum_{t =1}^n\left(Y_{it}Y'_t- \EE[Y_{it}Y'_t]\right) \right| \le  2\sqrt{6}\left({1\over n}\sum_{t =1}^nT_t^2(T'_t)^2\cdot {\log(d)\over n}\right)^{1/2}
			\end{equation}
			with probability greater than $1- 2d^{-3}$. Finally, note that 
			\begin{align}\label{eq_TjiTelli}\nonumber
			{1\over 4n}\sum_{i =1}^nT_{t}^2(T'_t)^2 & \le  {1\over n}\sum_{t =1}^n\left\{
			{\M_{it}\rho_j\M_{jt}\log^2(d) \over N^2} + {\rho_j^2\log^4(d) \over N^4} + {\rho_j\M_{jt} \log^3(d) \over N^3} + {\M_{it}\rho_j^2\log^3(d) \over N^3}
			\right\}\\\nonumber
			&=
			{\rho_j\Theta_{ij}\log^2(d) \over N^2} + {\rho_j^2\log^4(d) \over N^4} + {\rho_j\u_j \log^3(d) \over pN^3} + {\rho_j^2\u_i\log^3(d) \over pN^3}\\
			&\lesssim 	{\rho_j\a_i \psi_{jk}\log^2(d) \over pN^2} + {\rho_j\u_j \log^3(d) \over pN^3} + {\rho_j^2\u_i\log^3(d) \over pN^3}
			\end{align}
			by using (\ref{def_uag}) in the second equality and (\ref{lb_uL_mL}) and (\ref{eq_Theta_Ij}) to obtain the last line.  Finally, combining (\ref{eq_T2}) - (\ref{eq_TjiTelli}) gives
			\begin{align*}
				{\rm Rem_{13}}^{(jk)} &\lesssim {K\over \g_k}	\sqrt{\rho_j \psi_{jk} p\log^3(d) \over \ua_L nN^2} +{K\over \g_k}\sqrt{\rho_j{\u_j \over p}{ p^2\log^4(d) \over \ua_L^2nN^3}} + \rho_j\sqrt{p\log^4(d) \over \uu_L nN^3} + {p\rho_j \over \uu_L d^3}\\
				&\lesssim  {K\over \g_k}	\sqrt{\rho_j \psi_{jk} \log(d) \over  nN} +\sqrt{\rho_j{\u_j \over p}{K\log(d) \over \g_k nN}} + \rho_j\sqrt{p\log^4(d) \over \uu_L nN^3},
			\end{align*}
			for all $j,k$, with probability $1-6d^{-1}$. We also use  (\ref{lb_uL_mL}) and 
			$
				\ua_L\g_k \ge \ua_L\ug \ge c_0pK\log(d)/N.
			$
			This completes the proof. 
	\end{proof}
	
	 \subsection{Auxilliary lemmas}
	 In this section, we state three lemmas which are used in the main paper. The following lemma gives the range of $\lambda_{\min}(\wt Q\wt Q^\top )$ where $\wt Q = D_Q^{-1}Q$ and $Q = CA^\top $.
	\begin{lemma}\label{lem_QQ}
		Let $C = n^{-1}WW^\top $ and $M = D_W^{-1}CD_W^{-1}$. We have
		\[
		\left( 
		\min_{k\in [K], i\in I_k}A_{ik}^2
		\right)\lambda_{\min}(C)\lambda_{\min}(M) \le \lambda_{\min}(\wt Q \wt Q^\top ) \le \lambda_{\min}(M).
		\]
	\end{lemma}
	\begin{proof}
	Observe  that
		\[
			 D_Q = Q\1_p = CA^\top \1_p = C\1_K \overset{(\ref{disp_sum_C})}{=} D_W,
		\]
	whence
		\begin{equation}\label{lbd_min_QQ}
			\lambda_{\min}(\wt Q\wt Q^\top ) = \inf_{v\in \S^{K-1}} v^\top  D_W^{-1}CA^\top ACD_W^{-1} v.
		\end{equation}
		On the one hand, 
		\begin{align*}
		\inf_{v\in \S^{K-1}} v^\top  D_W^{-1}CA^\top ACD_W^{-1} v &\le \|C^{1/2}A^\top AC^{1/2}\|_{\rm{\rm{op}}} \inf_{v\in \S^{K-1}} v^\top  D_W^{-1}CD_W^{-1} v\\
		&= \|ACA^\top \|_{\rm{\rm{op}}}\cdot  \lambda_{\min}(M).
		\end{align*}
		The upper bound now follows using (\ref{col_sum_one})  and $\|ACA^\top \|_{\rm{\rm{op}}} \le  1$. The latter follows from the string of inequalities 
		\[
			\|ACA^\top \|_{\rm{\rm{op}}} \le \|ACA^\top \|_{\r} = \|ACA^\top \1_p\|_\i = \|AC\1_K\|_\i = {1\over n}\|\M \1_n\|_\i \le 1.
		\]
		The lower bound follows immediately from
		\begin{align*}
			\lambda_{\min}(\wt Q\wt Q^\top )  &\ge \lambda_{\min}(A^\top A)\lambda_{\min}(C)\lambda_{\min}(M)\\ 
			&\ge \lambda_{\min}(A_I^\top A_I)\lambda_{\min}(C)\lambda_{\min}(M)\\
			&\ge \left(\min_{k\in [K], i\in I_k}A_{ik}^2\right)\lambda_{\min}(C)\lambda_{\min}(M).
	  \end{align*}
	  	
  \end{proof}

	  \begin{lemma}\label{lem_lbd_min_C}
	  	Let $C = n^{-1}WW^\top $. We have
	  		\[
	  			\lambda_{\min}(C) \le {1\over K}.
	  		\]
	  \end{lemma}
	 \begin{proof}
	 		From the definition of the smallest eigenvalue, 
	 		\[
	 			\lambda_{\min}(C) = \inf_{v\in \S^{K-1}}v^\top Cv  \le \min_{1\le k\le K}C_{kk} = \min_{1\le k\le K}{1\over n}\|W_{k\cdot}\|^2.
	 		\]
	 		The result  follows from 
	 		\[
	 			{1\over n}\|W_{k\cdot}\|^2 \le {1\over n}\|W_{k\cdot}\|_1 \overset{(\ref{def_alpha_gamma})}{=} \g_k
	 		\] 
	 		and $$\min_k \g_k \le \sum_{k = 1}^K\g_k / K = 1/K.$$
	 \end{proof}

	\begin{lemma}\label{lem_mu_hat}
		Under condition (\ref{cond_Pi_I}), 
		\[
		\PP\left\{
		\left|  \min_{i\in L} (D_X)_{ii} - {\uu_L\over p} \right| \le 6\sqrt{\log(d) \over nN}
		\right\}\ge 1-2d^{-1}.
		\]
	\end{lemma}
	\begin{proof}
		For any $i\in L$, note that $(D_X)_{ii} - \u_i / p = n^{-1}\sum_{t = 1}^n\eps_{it}$. The  result follows from Lemma \ref{lem_t1} and the inequalities $\u_j / p \le 1$ for all $j\in[p]$ by (\ref{def_alpha_gamma}).
	\end{proof}

	For completeness, we state the well-known Bernstein  and Hoeffding inequalities for bounded random variables.
	\begin{lemma}[Bernstein's inequality for bounded random variables]\label{lem_bernstein}
		For independent random variables $Y_1, \ldots, Y_n$ with bounded ranges $[-B, B]$ and zero means, 
		\[
		\PP\left\{ {1\over n}\left|\sum_{i =1}^nY_i \right| > x \right\} \le 2\exp\left(-{n^2x^2 /2 \over v + nBx/3} \right),\qquad \text{for any }x\ge 0,
		\]
		where $v \ge var(Y_1 + \ldots + Y_n)$.
	\end{lemma}
	
	\begin{lemma}[Hoeffding's inequality] \label{hoeff}
		Let $Y_1,\ldots,Y_n$ be independent random variables with $\EE[Y_i]=0$ and $\PP\{ a_i\le Y_i \le b_i\}=1$. 
		For any $t\ge0$, we have
		\[ \PP\left\{ \left| \sum_{i=1}^n Y_i \right| > t \right\} \le 2\exp\left( -\frac{2t^2}{\sum_{i =1}^n(b_i-a_i)^2}  \right).\]
	\end{lemma}

	\bibliography{ref}
	
\end{document}